\documentclass[twoside]{article}

\usepackage[accepted]{aistats2023}
\usepackage{amsmath,amsfonts,amsthm,mathtools}
\usepackage[numbers,sort&compress]{natbib}
\usepackage[svgnames,table]{xcolor}
\usepackage{changepage}
\usepackage{subcaption, caption}
\usepackage{thm-restate}
\usepackage{tabularx,booktabs}
\usepackage{multirow, multicol}
\usepackage{listings}
\usepackage{pifont}
\usepackage{pgfplots,pgfplotstable}
\usepgfplotslibrary{groupplots}
\usepackage{hyperref}
\usepackage{cleveref}
\usepackage{wrapfig}
\usepackage[toc,page,header]{appendix}
\usepackage{minitoc}
\usepackage[framemethod=TikZ]{mdframed}

\hypersetup{
    colorlinks= true,
    citecolor = blue,
    linkcolor = blue
}

\mdfsetup{%
    middlelinecolor =   none,
    middlelinewidth =   1pt,
    roundcorner     =   5pt,
}

\newcommand{\quotes}[1]{``#1''}

\newcommand\independent{\protect\mathpalette{\protect\independenT}{\perp}}
\def\independenT#1#2{\mathrel{\rlap{$#1#2$}\mkern2mu{#1#2}}}

\newcommand{\veryshortarrow}[1][3pt]{\mathrel{%
   \hbox{\rule[\dimexpr\fontdimen22\textfont2-.2pt\relax]{#1}{.4pt}}%
   \mkern-4mu\hbox{\usefont{U}{lasy}{m}{n}\symbol{41}}}}

\newcommand{\into}[1]{{}^{\veryshortarrow}\hspace{-1mm} #1}
\newcommand{\out}[1]{{#1}^{\veryshortarrow}}

\makeatletter
\newcommand*\@dblLabelI {}
\newcommand*\@dblLabelII {}
\newcommand*\@dblequationAux {}

\def\@dblequationAux #1,#2,%
    {\def\@dblLabelI{\label{#1}}\def\@dblLabelII{\label{#2}}}

\newcommand*{\doubleequation}[3][]{%
    \par\vskip\abovedisplayskip\noindent
    \if\relax\detokenize{#1}\relax
       \let\@dblLabelI\@empty
       \let\@dblLabelII\@empty
    \else 
       \@dblequationAux #1,%
    \fi
    \makebox[0.5\linewidth-1.5em]{%
     \hspace{\stretch2}%
     \makebox[0pt]{$\displaystyle #2$}%
     \hspace{\stretch1}%
    }%
    \makebox[0.5\linewidth-1.5em]{%
     \hspace{\stretch1}%
     \makebox[0pt]{$\displaystyle #3$}%
     \hspace{\stretch2}%
    }%
    \makebox[3em][r]{(%
  \refstepcounter{equation}\theequation\@dblLabelI, 
  \refstepcounter{equation}\theequation\@dblLabelII)}%
  \par\vskip\belowdisplayskip
}
\makeatother

\definecolor{codegreen}{rgb}{0,0.6,0}
\definecolor{codegray}{rgb}{0.5,0.5,0.5}
\definecolor{codepurple}{rgb}{0.58,0,0.82}
\definecolor{backcolour}{rgb}{0.95,0.95,0.92}

\lstdefinestyle{mystyle}{
    backgroundcolor=\color{backcolour},   
    commentstyle=\color{codegreen},
    keywordstyle=\color{magenta},
    numberstyle=\tiny\color{codegray},
    stringstyle=\color{codepurple},
    basicstyle=\ttfamily\footnotesize,
    breakatwhitespace=false,         
    breaklines=true,                 
    captionpos=b,                    
    keepspaces=true,                 
    numbers=left,                    
    numbersep=5pt,
    showspaces=false,                
    showstringspaces=false,
    showtabs=false,                  
    tabsize=2
}

\newcommand{\cmark}{\ding{51}}
\newcommand{\xmark}{\textcolor{lightgray}{\ding{55}}}
\newcommand{\ccmark}{\textcolor{lightgray}{\ding{51}}}
\newcommand{\xxmark}{\ding{55}}

\newcolumntype{C}{>{\centering\arraybackslash}X}
\newcommand{\CC}[1]{\cellcolor{black!#1}}

\pgfplotsset{width=5cm,compat=1.9}

\theoremstyle{plain}
\newtheorem{theorem}{Theorem}[section]

\newtheorem{lemma}[theorem]{Lemma}

\theoremstyle{definition}
\newtheorem{definition}[theorem]{Definition}

\theoremstyle{remark}

\newtheorem{assum}{Assumption}
\crefformat{assum}{assum.~#2#1#3}
\Crefformat{assum}{Assum.~#2#1#3}
\crefrangeformat{assum}{assum.~#3#1#4 to~#5#2#6}
\Crefrangeformat{assum}{Assum.~#3#1#4 to~#5#2#6}

\crefformat{definition}{def.~#2#1#3}
\Crefformat{definition}{Def.~#2#1#3}
\crefrangeformat{definition}{defs.~#3#1#4 to~#5#2#6}
\Crefrangeformat{definition}{Defs.~#3#1#4 to~#5#2#6}

\crefformat{theorem}{theorem~#2#1#3}
\Crefformat{theorem}{Theorem~#2#1#3}

\begin{document}
\doparttoc
\faketableofcontents

%

%

\twocolumn[

\aistatstitle{To Impute or not to Impute? Missing Data in Treatment Effect Estimation}

\aistatsauthor{Jeroen Berrevoets \And Fergus Imrie \And  Trent Kyono}
\aistatsaddress{ University of Cambridge \And  UCLA  \And Meta } 
\aistatsauthor{James Jordon \And Mihaela van der Schaar}
\aistatsaddress{Alan Turing Institute \And University of Cambridge\\ Alan Turing Institute}
\runningauthor{Jeroen Berrevoets, Fergus Imrie, Trent Kyono, James Jordon, Mihaela van der Schaar}

]

\begin{abstract}
Missing data is a systemic problem in practical scenarios that causes noise and bias when estimating treatment effects. This makes treatment effect estimation from data with missingness a particularly tricky endeavour. A key reason for this is that standard assumptions on missingness are rendered insufficient due to the presence of an additional variable, treatment, besides the input (e.g. an individual) and the label (e.g. an outcome). The treatment variable introduces additional complexity with respect to {\it why} some variables are missing that is not fully explored by previous work. In our work we introduce {\it mixed confounded missingness} (MCM), a new missingness mechanism where some missingness {\it determines} treatment selection and other missingness {\it is determined by} treatment selection. Given MCM, we show that naively imputing all data leads to poor performing treatment effects models, as the act of imputation effectively {\it removes} information necessary to provide unbiased estimates. However, no imputation at all also leads to biased estimates, as missingness determined by treatment introduces bias in covariates. Our solution is {\it selective} imputation, where we use insights from MCM to inform precisely which variables should be imputed and which should not. We empirically demonstrate how various learners benefit from selective imputation compared to other solutions for missing data. We highlight that our experiments encompass both average treatment effects {\it and} conditional average treatment effects.
\end{abstract}

\section{Introduction}

Treatment effects are arguably the most important estimand in causal inference \citep{pearl2009causality, neyman1923applications, rubin1974estimating}, with a treatment {\it effect} lying at the heart of a causal question, and have been adopted in a wide range of fields such as medicine \citep{obermeyer2016predicting, alaa2021machine,berrevoets2020organite,berrevoets2021learning, bica2021real}, marketing \citep{devriendt2018literature, ascarza2018retention, debaere2019reducing}, and even human resources \citep{rombaut2020effectiveness}. Using causal inference, we try to more explicitly attribute effect to the treatment by carefully disentangling the role of the environment. We define {\it effect} as the difference in outcome when applying the treatment versus not applying the treatment (or any alternative treatment for that matter). 

Literature on inferring (or predicting) treatment effects is largely concerned with handling {\it selection bias}. That is, we identify a possible difference between the treated and non-treated subpopulations of the data, since treatment is rarely distributed uniformly random across the population. For example, cancer patients are assigned different oncology therapies based on their individual disease and patient characteristics. If not accounted for, selection bias leads to biased estimates. As such, many works focus on novel strategies to handle this bias \citep{johansson2016learning, alaa2017, shalit2017estimating, rubin1974estimating,imbens2015causal,hatt2022combining,zhang2022identifiable}.

However, these methods almost exclusively assume that data is complete. From a practitioner's standpoint, this may not always be the case and, in reality, {\it data is often incomplete} \citep{burton2004missing,little2012prevention}. 
Work on learning from incomplete data has focused on non-causal prediction problems and, in particular, on methods for imputing missing values \citep{MICE1,MissForest,raghunathan2000multivariate,nazabal2020handling,GAIN}.
One strategy adopted in the treatment effects literature is to use imputation methods to infer the missing variables before we train a learner on the data \citep{rubin1978multiple, rubin2004multiple,little2019statistical, kallus2018causal}.

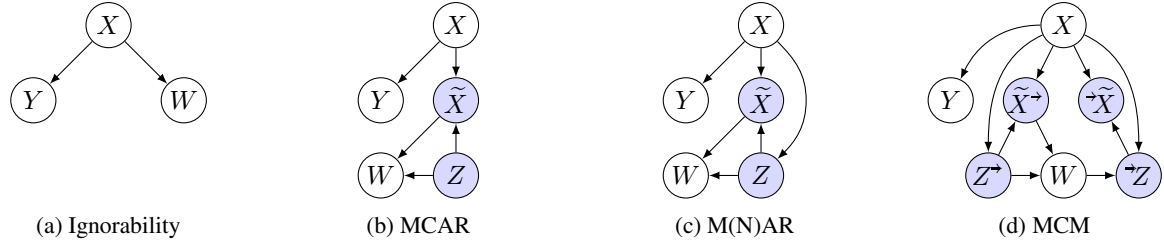
\begin{figure*}[t]
    \begin{subfigure}[b]{.23\textwidth}
        \centering
        \begin{tikzpicture}[
            roundnode/.style={circle, draw=black,  minimum size=6mm, inner sep=0},
            smallnode/.style={circle, draw=black,  minimum size=4mm, inner sep=0, fill=white},
            dummynode/.style={circle, draw=white,  minimum size=6mm, inner sep=0}
        ]
            \node[roundnode] (x) {$X$};d
            \node[roundnode] (w) at (1, -1) {$W$};
            \node[roundnode] (y) at (-1,-1) {$Y$};
            \node[dummynode] (dummy) at (0,-2) {};
        
            \draw[-latex] (x) -- (w);
            \draw[-latex] (x) -- (y);

    \end{tikzpicture}
    \caption{Ignorability}\label{fig:ignorability}
    \end{subfigure}%
    ~
    \begin{subfigure}[b]{0.23\textwidth}
    \centering
    \begin{tikzpicture}[
            roundnode/.style={circle, draw=black,  minimum size=6mm, inner sep=0},
            dummynode/.style={circle, draw=white,  minimum size=6mm, inner sep=0}
        ]
            \node[roundnode] (x)  {$X$};
            \node[roundnode, fill=blue!15] (x_tilde) at (0, -1) {$\widetilde{X}$};
            \node[roundnode] (y) at (-1, -1) {$Y$};
            \node[roundnode, fill=blue!15] (z) at (0, -2) {$Z$};
            \node[roundnode] (w) at (-1, -2) {$W$};
            \node[dummynode] (dummy) at (-1,0) {};

            \draw[-latex] (x) -- (x_tilde);
            \draw[-latex] (z) -- (x_tilde);
            \draw[-latex] (x) -- (y);
            \draw[-latex] (x_tilde) -- (w);
            \draw[-latex] (z) -- (w);
        \end{tikzpicture}
        \caption{MCAR}
        \label{fig:mcar}
    \end{subfigure}
    ~
    \begin{subfigure}[b]{0.23\textwidth}
    \centering
    \begin{tikzpicture}[
            roundnode/.style={circle, draw=black,  minimum size=6mm, inner sep=0},
            dummynode/.style={circle, draw=white,  minimum size=6mm, inner sep=0}
        ]
            \node[roundnode] (x) {$X$};
            \node[roundnode, fill=blue!15] (x_tilde) at (0, -1) {$\widetilde{X}$};
            \node[roundnode] (y) at (-1, -1) {$Y$};
            \node[roundnode, fill=blue!15] (z) at (0, -2) {$Z$};
            \node[roundnode] (w) at (-1, -2) {$W$};
            \node[dummynode] (dummy) at (-1,0) {};

            \draw[-latex] (x) -- (x_tilde);
            \draw[-latex] (x) -- (y);
            \draw[-latex] (x) to [out=315, in=90] (.6,-1) to [ out=270, in=45] (z);
            \draw[-latex] (z) -- (x_tilde);
            \draw[-latex] (x_tilde) -- (w);
            \draw[-latex] (z) -- (w);
        \end{tikzpicture}
        \caption{M(N)AR}
        \label{fig:mar}
    \end{subfigure}%
    ~
    \begin{subfigure}[b]{0.23\textwidth}
    \centering
    \begin{tikzpicture}[
            roundnode/.style={circle, draw=black,  minimum size=6mm, inner sep=0}
        ]
            \node[roundnode] (x) at (0,0) {$X$};
            \node[roundnode, fill=blue!15] (x_tilde_in) at (.5, -1) {$\into{\widetilde{X}}$};
            \node[roundnode, fill=blue!15] (z_in) at (1, -2) {$\into{Z}$};
            \node[roundnode, fill=blue!15] (x_tilde_out) at (-.5, -1) {$\out{\widetilde{X}}$};
            \node[roundnode, fill=blue!15] (z_out) at (-1, -2) {$\out{Z}$};
            \node[roundnode] (w) at (0, -2) {$W$};
            \node[roundnode] (y) at (-1.5, -1) {$Y$};

            
            \draw[-latex] (x) -- (x_tilde_in);
            \draw[-latex] (x) -- (x_tilde_out);
            
            \draw[-latex] (x_tilde_out) -- (w);
            \draw[-latex] (z_out) -- (w);
            
            \draw[-latex] (x) to [out=180, in=65] (y);
            \draw[-latex] (w) -- (z_in);
            \draw[-latex] (z_in) -- (x_tilde_in);
            \draw[-latex] (z_out) -- (x_tilde_out);
            \draw[-latex] (x) to [out=205, in=90] (z_out);
            \draw[-latex] (x) to [out=335, in=90] (z_in);
        \end{tikzpicture}
        \caption{MCM}
        \label{fig:mcm}
    \end{subfigure}
  
    \caption{{\bf (\subref{fig:ignorability}) Ignorability as a graphical model.} From \citet{richardson2013single}, we express ignorability as a DAG. For brevity, we drop the parentheses of $Y(w)$, as well as its accepted ``single world intervention path'' from $w$ to $Y(w)$; {\bf (\subref{fig:mcar}-\subref{fig:mcm}) DGP for missingness mechanisms.} Shaded nodes indicate nodes that relate to missing variables, white nodes relate to treatment effects. In \cref{fig:mcar} and \cref{fig:mar} we illustrate the standard causal extensions to MCAR and M(N)AR, respectively, including the treatment variable $W$. In \cref{fig:mcm} we illustrate MCM. Unique to MCM is to allow treatment to both \textit{cause} missingness (through $\into{\widetilde{X}}$ and $\into{Z}$) and \textit{be caused by} missingness (corresponding to $\out{\widetilde{X}}$ and $\out{Z}$).}
    \label{fig:missingness}
    \rule{\linewidth}{.75pt}
    \vspace{-5mm}
\end{figure*}

Missing data is more challenging in treatment effects compared to prediction problems due to the presence of a treatment variable.
First, 
the fact that certain variables are missing may contribute to which treatment is administered, contributing to selection bias. For example, being deprived of information, a clinician may opt for a less risky treatment, just to be safe.
Thus, imputation effectively {\it removes information}--- i.e. given that missingness may cause treatment selection, it is much harder to handle selection bias when the information about which variables were missing is removed. 
This observation has led to many \textit{not imputing} missing values and instead considering a missing variable as another value and using it directly \citep{mayer2020doubly, mayer2020missdeepcausal, d2000estimating, rosenbaum1984reducing}.

Conversely, it may also be the case that the treatment decision will impact which variables are measured, i.e. missingness may be {\it caused by the treatment}. We believe this is a ubiquitous but overlooked scenario.
For example, in medicine, many drugs require baseline blood tests to be performed before treatment commences \cite{drugmonitoring}. However, these blood tests typically \textit{would not} be performed if a different treatment decision was made.
Missingness as a result of the treatment choice increases the difference between treatment subpopulations, even when there was no difference to begin with, had all patient covariates been observed. 
Adjusting on these missing variables therefore {\it introduces bias} to the model.
This effect has \textit{not} previously been considered in existing missingness mechanisms. However, in scenarios such as healthcare, this is almost always the case. Thus, to impute or not to impute?


\vspace{5pt}\begin{mdframed}[backgroundcolor=black!5]
\textbf{Example: knee pain.} Let us consider a specific example of MCM in practice (we provide an additional example in \cref{app:example}).

Around one in six general practice consultations in England are for arthritis and musculoskeletal issues \cite{Helliwell2014}, of which knee pain is one of the most common complaints.
Treatment plans often consist of pharmaceutical-based pain management, physiotherapy, or both \cite{Alshami2014}.
However, surgery is considered depending on patient characteristics (e.g. young, active) and severity of their condition. 
Patients who are treated with pain relief medication or physiotherapy will often not have additional tests performed (such as radiography or magnetic resonance imaging). However, these tests will almost always be performed before surgery, together with other routine preoperative tests, affecting which patient covariates are collected.

Furthermore, one common surgical intervention is knee arthroscopy \cite{Shah2018}, a gold standard for simultaneous diagnosis and treatment of knee disorders \cite{Rossbach2014}. 
As a surgical procedure allowing doctors to examine the knee, knee arthroscopy provides additional information that is not captured if a different treatment plan is chosen.
While specific to one indication, our two medical examples clearly illustrate the complexity of missingness patterns in real-world observational medical data, motivating and reinforcing the need for, and relevance of, MCM. 
\end{mdframed}

{\bf Contribution.} 
We introduce (and motivate adoption of) a formal description of missingness in data used to estimate treatment effects \citep{neyman1923applications, rubin1974estimating, mohan2013missing,}. In particular, we find that previous attempts--- dating as far back as the 1980s \citep[Appendix B]{rosenbaum1984reducing} ---at formalising missingness in treatment effects allow for inaccurate descriptions of missingness and its impact. We illustrate why these descriptions are insufficient and provide an alternative termed {\it mixed confounded missingness} (MCM). We argue that MCM is a general-purpose missingness mechanism, distinct from well-known missingness mechanisms such as {\it missing (completely) at random} \citep{little2019statistical}, and a refinement of {\it conditional independence of treatment} \citep{rosenbaum1984reducing} that should be adopted to describe missingness when estimating treatment effects.

Based on the insights provided by MCM, we propose a strategy for handling missing data in treatment effects, termed selective imputation. Our approach is theoretically motivated and we provide empirical evidence of how methods benefit from this approach and demonstrate the harm when missingness is not correctly dealt with.

\section{Preliminaries} \label{sec:prelims}
\vspace{-5pt}

Estimating causal effects requires us to answer a {\it counterfactual} question. In particular, when we observe the outcome after applying a treatment on an individual, it is impossible to also observe that individual's outcome under alternative treatment \citep{holland1986statistics}. As a treatment effect is defined as the difference between both outcomes, we are tasked with inferring an estimand which is {\it never observed}, which is crucially different from standard supervised learning.

One can estimate causal effects by conducting randomised controlled trials (RCTs) \citep{fisher1925, neyman1923applications}. However, RCTs are often very expensive, and are sometimes considered unethical in a clinical setting \citep{hellman1991mice,edwards1999ethical}. However, the alternative we consider in our work--- estimating effects from observational data ---comes with its own challenges. In particular, comparing each treatment's subpopulation in an {\it observational} dataset will result in biased estimates. If treatment is not assigned randomly, but instead based on an individual's characteristics, with different outcomes, then these individuals are more represented in each subpopulation as a result, and estimating outcomes become biased. This phenomenon is often termed {\it selection bias}.

{\bf Notation.} Let $X\in\mathcal{X}\subseteq\mathbb{R}^d$ be the covariates of an individual; let the individual be treated with $W\in\{0,1\}$; and let $Y \in \mathcal{Y} \subseteq \mathbb{R}$ be their observed outcome. Practically, $X$ could be a patient with lung cancer; $W=1$ could be chemotherapy (and $W=0$ radiotherapy); and $Y$ their tumour size after treatment. We use a subscript, $X_i$ to indicate the $i$\textsuperscript{th} element in $X$, which means that $X_i \in \mathbb{R}$.

{\bf Assumptions in causal inference.} Estimating unbiased treatment effects from observational data has received a lot of attention in recent years. One of the more popular avenues in the literature is the potential outcomes (POs) framework of causality \citep{neyman1923applications, rubin1974estimating}. We define the PO of a treatment $w\in\{0,1\}$ as $Y(w)$, where $Y(w)$ corresponds to the outcome an individual would have experienced had they been assigned treatment $W=w$. While the standard consistency assumption (see \Cref{assumption:consistency} below) allows us to interpret the observed outcome as the potential outcome of the observed treatment, i.e. $Y = Y(W)$; selection bias makes estimating the counterfactual outcome, $Y(\neg W)$, more involved. 

{\bf Goal.} Ultimately, our goal is to correctly identify the treatment effect, for which we need to account for bias in the data. Countering selection bias is achieved by correctly adjusting for the confounders. To do so, we make the following assumptions, which are widely considered as standard in the PO-framework:
\begin{assum}[Consistency]
\label{assumption:consistency}
The observed outcome $Y_i = Y(W_i) = Y(w)$ if $W_i = w, $ for $w\in \{0,1\}$ and $i=1,2,\dots,N$, \footnote{The well-known \emph{stable unit treatment value assumption} (SUTVA) assumes both no interference and consistency \citep{rubin80comment}. The equation in our consistency assumption also implies no interference.} i.e. outcomes in the data correspond to one of the potential outcomes.
\end{assum}
\begin{assum}[Ignorability] \label{assum:ignorability}
The joint distribution $p(X,W,Y)$ satisfies strong ignorability: $Y(0),Y(1) \independent W | X$, i.e. the potential outcomes are independent of the treatment, conditioned on $X$, implying that there are no {\it additional} (unobserved) confounders beyond the variables in $X$. 
\end{assum}
\begin{assum}[Overlap]\label{assum:overlap}
The distribution $p(X,W,Y)$ satisfies overlap: $\exists~\delta\in (0,1)$ s.t. $ \delta < p(W | X=x) < 1-\delta, \forall x\in \mathcal{X}$, i.e. each individual has a non-zero probability to receive either treatment.
\end{assum}

{\bf Graphical models and causality.} Alternatively to POs, we can express causal relationships as a graphical model. In particular, a causal relationship is depicted as a directed edge in a directed acyclic graph (DAG) \citep{pearl2009causality}, where a parent node is the cause and the child node is the effect. The ignorability assumption in \Cref{assum:ignorability} is sometimes illustrated in such a graphical model \citep{richardson2013single}. Specifically, the ignorability assumption can be expressed as the DAG shown in \cref{fig:ignorability}. Typically the influence of the treatment on the outcome is expressed as a {\it single world intervention graph} (SWIG): \begin{tikzpicture}[
    roundnode/.style={circle, draw=black,  minimum size=5mm, inner sep=0},
    smallnode/.style={circle, draw=black,  minimum size=4mm, inner sep=0, fill=white},
    baseline=-1mm
]
    \node[roundnode] (w) {$W$};
    \node[smallnode] (w_) at (.4, 0) {$w$};
    \node[roundnode] (y) at (1.5, 0) {$Y$};

    \draw[-latex] (w_) -- (y);
\end{tikzpicture} \citep{richardson2013single}. We have removed this SWIG-path from our figures in order to focus our discussion on the path(s) between $X$ and $W$.

Note that the set of DAGs (i.e. the Markov equivalence class) that satisfy \Cref{assum:ignorability} encompasses more than just the DAG in \cref{fig:ignorability}. Other DAGs in this equivalence class can just as easily respect \Cref{assumption:consistency,assum:overlap,assum:ignorability}, but they would make no sense. For example, reversing the arrow between $X$ and $Y$ would imply that outcome causes the covariates in the individuals, while still respecting ignorability. Instead, \cref{fig:ignorability} is motivated through logical reasoning, where treatment and outcome are caused by the covariates. Throughout the remainder of this paper, we will build (and extend) heavily on \cref{fig:ignorability}, employing similar logical reasoning.

{\bf Causal estimands.} We now arrive at our two estimands of interest: the {\it average treatment effect} (ATE), and the {\it conditional average treatment effect} (CATE). Given the notation above, we can define each estimand as follows:
\begin{definition}[ATE] \label{eq:ATE}
The ATE is defined as the population-wide difference between a treatment's potential outcomes. Mathematically, we can define the ATE as follows: $\bar\tau(\mathcal{X}) \coloneqq \mathbb{E}_{\mathcal{X}}[Y(1) - Y(0)]$.
\end{definition}
\begin{definition}[CATE] \label{eq:CATE}
The CATE is defined as a conditional difference between a treatment's potential outcomes. Mathematically, we can define the CATE as follows: $\tau(x) \coloneqq \mathbb{E}[Y(1) - Y(0) | X = x]$.
\end{definition}

{\bf Missingness.} In practice, a sample $X$ may be incomplete. For example, a clinician responding to an urgent trauma case may have to select treatment based on incomplete information. In this scenario, the incomplete variables are considered {\it missing}. To learn from these data we could consider {\it completing} this sample through imputation, but missingness in itself may be informative. Perhaps the clinician's decision would have been different if they had had complete information. Typically, we define three {\it mechanisms} to describe how a variable ended up to be missing: a first is {\it missing completely at random} (MCAR), where missingness in one variable is independent of the other variables as well as its own value, a second is {\it missing at random} (MAR), where the missingness in one variable may depend on the other (observed) variables, and the third is {\it missing not at random} (MNAR), where missingness is typically assumed to be caused by variables outside the observed covariates \citep{vanbuuren-MissingData, rubin1976inference, little2019statistical}. Generally, MCAR is attributed to noise in data collection.

We indicate the missing data in $X$ with a variable $Z \in \{1,\star\}^d$, where $Z_i = \star$ if $X_i$ is missing, and $Z_i = 1$ if $X_i$ is observed. Having $Z \independent X, Y(w)$, corresponds with MCAR; and $Z \not\independent X$, but $Z \independent Y(w) | X$, corresponds with M(N)AR. In our paper, we denote the complete (but unobserved) sample as $X$, and the incomplete (but observed) sample as $\widetilde{X} \coloneqq Z \odot X \in \{\star, \mathbb{R}\}^d$ where we define $\odot$ as the element-wise product and we take $X_i \times \star = \star$, i.e. if $X_i$ is unobserved it equals $\star$ in $\widetilde{X}$.

Note that none of the existing missingness mechanisms take the treatment $W$ into account as standard, since they span a broader literature beyond treatment effects. We have illustrated MCAR and M(N)AR as DAGs in \cref{fig:mcar} and \cref{fig:mar}, respectively, where we have included arrows from $\widetilde{X}$ and $Z$ to $W$. This corresponds to the situation where treatment is decided based on what is actually observed; the alternative where $X$ is causing $W$ would lead to a confounded setting, which is commonly assumed not to be the case in the literature. Note that these DAGs respect the conditional independent statements assumed in their respective missingness mechanisms. As we have noted in our introduction, we could either impute a missing value $X_{ij}$ with an estimate thereof, denoted $\dot{X}_{ij}$, such that we can predict a treatment effect from data with imputed samples (i.e. learn $\bar\tau(\dot{X})$ or $\tau(\dot{X})$); or we could predict treatment effects from data with missingness directly (i.e. learn $\bar\tau(\widetilde{X})$ or $\tau(\widetilde{X})$).

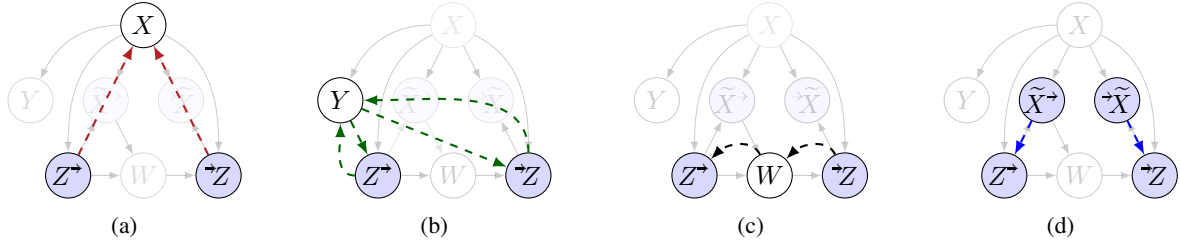
\begin{figure*}[t]
    \begin{subfigure}[b]{0.23\textwidth}
        \centering
        \begin{tikzpicture}[
                roundnode/.style={circle, draw=black,  minimum size=6mm, inner sep=0}
            ]
            \node[roundnode] (x) at (0,0) {$X$};
            \node[roundnode, fill=blue!15, opacity=.10] (x_tilde_in) at (.5, -1) {$\into{\widetilde{X}}$};
            \node[roundnode, fill=blue!15] (z_in) at (1, -2) {$\into{Z}$};
            \node[roundnode, fill=blue!15, opacity=.10] (x_tilde_out) at (-.5, -1) {$\out{\widetilde{X}}$};
            \node[roundnode, fill=blue!15] (z_out) at (-1, -2) {$\out{Z}$};
            \node[roundnode, opacity=.15] (w) at (0, -2) {$W$};
            \node[roundnode, opacity=.15] (y) at (-1.5, -1) {$Y$};

            
            \draw[-latex, opacity=.20] (x) -- (x_tilde_in);
            \draw[-latex, opacity=.20] (x) -- (x_tilde_out);
            
            \draw[-latex, opacity=.20] (x_tilde_out) -- (w);
            \draw[-latex, opacity=.20] (z_out) -- (w);
            
            \draw[-latex, opacity=.20] (x) to [out=180, in=65] (y);
            \draw[-latex, opacity=.20] (w) -- (z_in);
            \draw[-latex, opacity=.20] (z_in) -- (x_tilde_in);
            \draw[-latex, opacity=.20] (z_out) -- (x_tilde_out);
            \draw[-latex, opacity=.20] (x) to [out=205, in=90] (z_out);
            \draw[-latex, opacity=.20] (x) to [out=335, in=90] (z_in);
            
            \draw[-latex, color=FireBrick, dashed, thick] (z_out) -- (x);
            \draw[-latex, color=FireBrick, dashed, thick] (z_in) -- (x);
            
        \end{tikzpicture}
        \caption{ }
        \label{fig:exhaustive:z_to_x}
    \end{subfigure}%
    ~
    \begin{subfigure}[b]{0.23\textwidth}
        \centering
        \begin{tikzpicture}[
                roundnode/.style={circle, draw=black,  minimum size=6mm, inner sep=0}
            ]
            \node[roundnode, opacity=.1] (x) at (0,0) {$X$};
            \node[roundnode, fill=blue!15, opacity=.10] (x_tilde_in) at (.5, -1) {$\into{\widetilde{X}}$};
            \node[roundnode, fill=blue!15] (z_in) at (1, -2) {$\into{Z}$};
            \node[roundnode, fill=blue!15, opacity=.10] (x_tilde_out) at (-.5, -1) {$\out{\widetilde{X}}$};
            \node[roundnode, fill=blue!15] (z_out) at (-1, -2) {$\out{Z}$};
            \node[roundnode, opacity=.20] (w) at (0, -2) {$W$};
            \node[roundnode] (y) at (-1.5, -1) {$Y$};

            
            \draw[-latex, opacity=.20] (x) -- (x_tilde_in);
            \draw[-latex, opacity=.20] (x) -- (x_tilde_out);
            
            \draw[-latex, opacity=.10] (x_tilde_out) -- (w);
            \draw[-latex, opacity=.20] (z_out) -- (w);
            
            \draw[-latex, opacity=.20] (x) to [out=180, in=65] (y);
            \draw[-latex, opacity=.20] (w) -- (z_in);
            \draw[-latex, opacity=.20] (z_in) -- (x_tilde_in);
            \draw[-latex, opacity=.10] (z_out) -- (x_tilde_out);
            \draw[-latex, opacity=.20] (x) to [out=205, in=90] (z_out);
            \draw[-latex, opacity=.20] (x) to [out=335, in=90] (z_in);

            \draw[-latex, color=DarkGreen, dashed, thick] (z_out) to [out=180, in=270] (y);
            \draw[-latex, color=DarkGreen, dashed, thick] (z_in) to [out=90, in=0] (y);
            \draw[-latex, color=DarkGreen, dashed, thick] (y) -- (z_out);
            \draw[-latex, color=DarkGreen, dashed, thick] (y) -- (z_in);

        \end{tikzpicture}
        \caption{ }
        \label{fig:exhaustive:y_to_z}
    \end{subfigure}
    ~
    \begin{subfigure}[b]{0.23\textwidth}
        \centering
        \begin{tikzpicture}[
                roundnode/.style={circle, draw=black,  minimum size=6mm, inner sep=0}
            ]
            \node[roundnode, opacity=.1] (x) at (0,0) {$X$};
            \node[roundnode, fill=blue!15, opacity=.20] (x_tilde_in) at (.5, -1) {$\into{\widetilde{X}}$};
            \node[roundnode, fill=blue!15] (z_in) at (1, -2) {$\into{Z}$};
            \node[roundnode, fill=blue!15, opacity=.20] (x_tilde_out) at (-.5, -1) {$\out{\widetilde{X}}$};
            \node[roundnode, fill=blue!15] (z_out) at (-1, -2) {$\out{Z}$};
            \node[roundnode] (w) at (0, -2) {$W$};
            \node[roundnode, opacity=.20] (y) at (-1.5, -1) {$Y$};

            
            \draw[-latex, opacity=.20] (x) -- (x_tilde_in);
            \draw[-latex, opacity=.20] (x) -- (x_tilde_out);
            
            \draw[-latex, opacity=.20] (x_tilde_out) -- (w);
            \draw[-latex, opacity=.20] (z_out) -- (w);
            
            \draw[-latex, opacity=.20] (x) to [out=180, in=65] (y);
            \draw[-latex, opacity=.20] (w) -- (z_in);
            \draw[-latex, opacity=.20] (z_in) -- (x_tilde_in);
            \draw[-latex, opacity=.20] (z_out) -- (x_tilde_out);
            \draw[-latex, opacity=.20] (x) to [out=205, in=90] (z_out);
            \draw[-latex, opacity=.20] (x) to [out=335, in=90] (z_in);

            \draw[-latex, color=Black, dashed, thick] (z_in) to [out=115, in=45 ] (w);
            \draw[-latex, color=Black, dashed, thick] (w) to [out=115, in=45 ] (z_out);
        \end{tikzpicture}
        \caption{ }
        \label{fig:exhaustive:z_to_w}
    \end{subfigure}%
    ~
    \begin{subfigure}[b]{0.23\textwidth}
        \centering
        \begin{tikzpicture}[
                roundnode/.style={circle, draw=black,  minimum size=6mm, inner sep=0}
            ]
            \node[roundnode, opacity=.20] (x) at (0,0) {$X$};
            \node[roundnode, fill=blue!15] (x_tilde_in) at (.5, -1) {$\into{\widetilde{X}}$};
            \node[roundnode, fill=blue!15] (z_in) at (1, -2) {$\into{Z}$};
            \node[roundnode, fill=blue!15] (x_tilde_out) at (-.5, -1) {$\out{\widetilde{X}}$};
            \node[roundnode, fill=blue!15] (z_out) at (-1, -2) {$\out{Z}$};
            \node[roundnode, opacity=.20] (w) at (0, -2) {$W$};
            \node[roundnode, opacity=.20] (y) at (-1.5, -1) {$Y$};

            
            \draw[-latex, opacity=.20] (x) -- (x_tilde_in);
            \draw[-latex, opacity=.20] (x) -- (x_tilde_out);
            
            \draw[-latex, opacity=.20] (x_tilde_out) -- (w);
            \draw[-latex, opacity=.20] (z_out) -- (w);
            
            \draw[-latex, opacity=.20] (x) to [out=180, in=65] (y);
            \draw[-latex, opacity=.20] (w) -- (z_in);
            \draw[-latex, opacity=.20] (z_in) -- (x_tilde_in);
            \draw[-latex, opacity=.20] (z_out) -- (x_tilde_out);
            \draw[-latex, opacity=.20] (x) to [out=205, in=90] (z_out);
            \draw[-latex, opacity=.20] (x) to [out=335, in=90] (z_in);

            \draw[-latex, color=Blue, dashed, thick] (x_tilde_out) -- (z_out);
            \draw[-latex, color=Blue, dashed, thick] (x_tilde_in) -- (z_in);
        \end{tikzpicture}
        \caption{ }
        \label{fig:exhaustive:x_to_z}
    \end{subfigure}%
  
    \caption{{\bf Arrows that are not included in MCM.} Considering each possible direct arrow from, and to, the missingness variables ($\into{Z}$ and $\out{Z}$), should result in 20 arrows with five remaining variables. Excluding the paths that {\it are} included, and the paths {\it across} missingness indicators, we are left with 10 paths that are {\it seemingly} missing from our definition. From \cref{fig:exhaustive:z_to_x,fig:exhaustive:y_to_z,fig:exhaustive:z_to_w,fig:exhaustive:x_to_z}, we depict each missing arrow in function of one variable: $X$, then $Y(w)$, $W$, and then $\out{\widetilde{X}}$, $\into{\widetilde{X}}$, respectively.}
    \label{fig:exhaustive:overview}
    \rule{\linewidth}{.75pt}
    
\end{figure*}

\vspace{-8pt}
\section{Mixed Confounded Missingness (MCM)}\label{sec:MCM}

Some missingness may be {\it caused by} treatment, while other missingness may {\it cause} treatment--- a situation that cannot be modelled by the MCAR nor the M(N)AR missingness mechanisms. This has important consequences as missingness is now a mixture of confounding and non-confounding elements, leading us to term our proposal: {\it mixed confounded missingness} (MCM). We have illustrated MCM as a DAG in \cref{fig:mcm}.

In this section, we explain why there are no other arrows included in MCM, i.e. motivate that it is complete and general. Then, we compare MCM to previous proposals for missingness in CATE.

Like MCAR and M(N)AR, MCM describes the interactions between $Z$ and the remaining variables, $X$, $W$, $Y(w)$, (and $\widetilde{X}$). However, contrasting MCAR and M(N)AR, MCM specifically models various possible interactions with missingness by splitting $Z$ into two distinct factors, $\out{Z}$ and $\into{Z}$. The former captures missingness causing treatment, and the latter missingness caused by treatment. Acknowledging the possibility of missingness that causes or is caused by treatment, leads to the following assumption:
\begin{assum}[Missingness factors in MCM]\label{assum:mcm}
We assume there exists a partition of $Z = \{\out{Z}, \into{Z}\}$ s.t. $\out{Z} \independent \into{Z} | W, X$, further implying that $\out{\widetilde{X}} \independent \into{\widetilde{X}} | W, X$ as $\out{Z} \independent \into{\widetilde{X}} | X, W$ and $\out{\widetilde{X}} \independent \into{Z} | X, W$, where $\into{\widetilde{X}}$ and $\out{\widetilde{X}}$ are the covariate sets with missingness corresponding to $\into{Z}$ and $\out{Z}$, respectively.
\end{assum}

{\bf MCM is exhaustive.} Being exhaustive is important, as through exhaustiveness we can argue that MCM covers {\it all} possible scenarios that can lead to missing data with treatments. For our discussion, exhaustiveness thus means that each potential arrow is either included in MCM, or its absence in MCM has a compelling argument.

Five variables ($X$, $\into{\widetilde{X}}$, $\out{\widetilde{X}}$, $W$, and $Y$) that are fully connected to $\into{Z}$ and $\out{Z}$ should result in 20 distinct paths (four for each of the five variables). However, \cref{fig:mcm} connects only 6 paths to either $\into{Z}$ or $\out{Z}$, leaving 14 paths unaccounted for. With \cref{assum:mcm}, we can exclude paths across missingness factors (i.e. the paths that directly connect $\out{Z}$ and $\into{\widetilde{X}}$, and those that directly connect $\into{Z}$ to $\out{\widetilde{X}}$), reducing the number of unaccounted for paths to 10.

We will now argue why these 10 paths (in \cref{fig:exhaustive:overview}) are not included in MCM. First, we discuss \cref{fig:exhaustive:z_to_x}. Having
\begin{tikzpicture}[
    roundnode/.style={circle, draw=black,  minimum size=5mm, inner sep=0},
    baseline=-1mm
]
    \node[roundnode, fill=blue!15] (z_out) at (-1, 0) {$\out{Z}$};
    \node[roundnode] (x) {$X$};
    \node[roundnode, fill=blue!15] (z_in) at (1, 0) {$\into{Z}$};

    \draw[-latex, dashed, color=FireBrick] (z_out) -- (x);
    \draw[-latex, dashed, color=FireBrick] (z_in) -- (x);
\end{tikzpicture} allows covariates to change {\it depending on} what other variables are missing, implying that $X$ is dependent on whatever dataset an individual is included in. Having this dependence would imply that $X$, representing the fully observed-- {\it true} --set of covariates, would be different across datasets, despite representing the same individual. Furthermore, including the paths in \cref{fig:exhaustive:z_to_x} introduces a cycle: 
\begin{tikzpicture}[
    roundnode/.style={circle, draw=black,  minimum size=5mm, inner sep=0},
    baseline=-1mm
]
    \node[roundnode, fill=blue!15] (z_in) at (1.2, 0) {$\into{Z}$};
    \node[roundnode] (x) {$X$};

    \draw[-latex, dashed, color=FireBrick, looseness=.9] (z_in) to [out=202, in=338] (x);
    \draw[-latex, looseness=.9] (x) to [out=22, in=158] (z_in);
\end{tikzpicture} (and similarly for $\out{Z}$), which violates the DAG structure.

In \cref{fig:exhaustive:y_to_z} we immediately observe cycles, meaning that only two of the four presented arrows can exist simultaneously. Let us first consider the paths where $Y(w)$ is causing missingness, i.e. 
\begin{tikzpicture}[
    roundnode/.style={circle, draw=black,  minimum size=5mm, inner sep=0},
    baseline=-1mm
]
    \node[roundnode, fill=blue!15] (z_out) at (-1, 0) {$\out{Z}$};
    \node[roundnode] (y) {$Y$};
    \node[roundnode, fill=blue!15] (z_in) at (1, 0) {$\into{Z}$};

    \draw[latex-, dashed, color=DarkGreen] (z_out) -- (y);
    \draw[latex-, dashed, color=DarkGreen] (z_in) -- (y);
\end{tikzpicture}. In the potential outcomes setting, $Y(w)$ is topologically last in the DAG; take \cref{fig:ignorability}, where inclusion upon of the single world intervention path, 
\begin{tikzpicture}[
    roundnode/.style={circle, draw=black,  minimum size=5mm, inner sep=0},
    smallnode/.style={circle, draw=black,  minimum size=4mm, inner sep=0, fill=white},
    baseline=-1mm
]
    \node[roundnode] (w) {$W$};
    \node[smallnode] (w_) at (.4, 0) {$w$};
    \node[roundnode] (y) at (1.5, 0) {$Y$};

    \draw[-latex] (w_) -- (y);
\end{tikzpicture}, we clearly see that no variable is caused by $Y(w)$. Having $Y(w)$ as the final observation makes sense; once we observe the outcome from a treatment, the covariate observations are left untouched. The absence of the remaining two arrows,
\begin{tikzpicture}[
    roundnode/.style={circle, draw=black,  minimum size=5mm, inner sep=0},
    baseline=-1mm
]
    \node[roundnode, fill=blue!15] (z_out) at (-1, 0) {$\out{Z}$};
    \node[roundnode] (y) {$Y$};
    \node[roundnode, fill=blue!15] (z_in) at (1, 0) {$\into{Z}$};

    \draw[-latex, dashed, color=DarkGreen] (z_out) -- (y);
    \draw[-latex, dashed, color=DarkGreen] (z_in) -- (y);
\end{tikzpicture}, is similarly argued as the absence of the arrows in \cref{fig:exhaustive:z_to_x}. Namely, if $Z$ were to directly influence outcome, then an individual represented in two different datasets--- with different missing variables ---would have conflicting outcomes. Clearly, one person can only have one outcome \citep{holland1986statistics}, i.e. the arrows in \cref{fig:exhaustive:y_to_z} cannot exist.

Next, \cref{fig:exhaustive:z_to_w}, illustrating the existing arrows, {\it reversed}. Besides these arrows resulting in cycles--- specifically,  
\begin{tikzpicture}[
    roundnode/.style={circle, draw=black,  minimum size=5mm, inner sep=0},
    baseline=-1mm
]
    \node[roundnode, fill=blue!15] (z_out) at (-1, 0) {$\out{Z}$};
    \node[roundnode] (w) {$W$};
    \node[roundnode, fill=blue!15] (z_in) at (1, 0) {$\into{Z}$};

    \draw[-latex, dashed, color=Black] (w) to [out=158, in=22] (z_out);
    \draw[-latex] (z_out) to [out=338, in=202] (w);
    \draw[-latex, dashed, color=Black] (z_in) to [out=158, in=22] (w);
    \draw[-latex] (w) to [out=338, in=202] (z_in);
\end{tikzpicture} ---they also violate our definition for $\into{Z}$ and $\out{Z}$. In particular, the {\it reason} for these distinct factors is precisely their respective directed paths from and to $W$.

Lastly, we consider \cref{fig:exhaustive:x_to_z} depicting the reverse arrows from $\into{\widetilde{X}}$ to $\into{Z}$, and similarly from $\out{\widetilde{X}}$ to $\out{Z}$. Here too, besides the obvious cycles, 
\begin{tikzpicture}[
    roundnode/.style={circle, draw=black,  minimum size=5mm, inner sep=0},
    baseline=-1mm
]
    \node[roundnode, fill=blue!15] (z_in) at (-1, 0) {$\into{Z}$};
    \node[roundnode] (x) {$\into{\widetilde{X}}$};

    \draw[-latex, dashed, color=Blue] (x) to [out=158, in=22] (z_in);
    \draw[-latex] (z_in) to [out=338, in=202] (x);
\end{tikzpicture} and similarly for $\out{\widetilde{X}}$ and $\out{Z}$, $\widetilde{X}$ is deterministically defined an element-wise product between $Z$ and $X$. This relationship is unambiguous. If indeed $\into{Z}$ and $\out{Z}$ are caused by the fully observed covariates $X$, it seems almost silly to consider them to be also caused by the partially observed covariates $\widetilde{X}$, which represents almost {\it exactly} the same entity. In fact, their only difference is completely captured in $Z$. With this, we can safely remove the arrows in \cref{fig:exhaustive:overview}, resulting in MCM (with a total of 6 arrows in $Z$) being exhaustive.

{\bf Missing data and existing assumptions.} Having argued each arrow in \cref{fig:mcm}, we will now relate MCM to previous descriptions of missingness in the treatment effects setting. Introduced in \citet{rosenbaum1984reducing} and further investigated in \citet{mattei2009estimating, mayer2020doubly} and \citet{blake2020estimating}, we find that typically the ignorability assumption in the treatment effects literature (that is, \cref{assum:ignorability}) is extended to the setting with missingness by simply replacing the condition in \cref{assum:ignorability} by,
\begin{equation} \label{eq:CATE:assum:UDM}
    Y(0), Y(1) \independent W | \tilde{X}, Z,
\end{equation}
which translates to {\it ``unconfoundedness despite missingness''}\footnote{Note that, we do not {\it have to} explicitly include $Z$ in the condition (see for example \citet{mayer2020doubly}) as $\widetilde{X}$ is related in a deterministic way to $Z$. However, we chose to include $Z$ in these conditions, as they are related to the DAG presented in \cref{fig:mcm}, where we have also included $Z$. Later in our paper, we will remove $Z$ from our DAGs and subsequently our conditions, in order to focus on the relationships with $\widetilde{X}$ directly.}. To verify \cref{eq:CATE:assum:UDM}, typically the above unconfoundedness assumption is combined with one of the two assumptions below:
\begin{align}
    \text{CIT:}&\quad W \independent X | \tilde{X}, Z \label{eq:CIT}\\
    \quad {\bf or} \nonumber\\
    \text{CIO:}&\quad Y(0), Y(1) \independent X | \tilde{X}, Z, \label{eq:CIO}
\end{align}
where CIT in \cref{eq:CIT} stands for {\it conditional independence of treatment}, and CIO in \cref{eq:CIO} stands for {\it conditional independence of outcome} \citep{mayer2020doubly}. Essentially, CIT and CIO assume there is no additional information in the fully observed $X$ compared to the observed $\tilde{X}$ and $Z$ to predict treatment and outcome, respectively--- i.e. adjustment for $\widetilde{X}$ (and $Z$) still warrants ignorability.

The assumptions depicted in \cref{eq:CIT,eq:CIO} can be considered a {\it logical consequence} of the assumption in \cref{eq:CATE:assum:UDM}. In particular, for the potential outcomes to be ignorable from treatment, as is assumed through \cref{eq:CATE:assum:UDM}, there cannot exist direct arrows beyond those in $\widetilde{X}$ from $X$ to either $W$ (which is implied by \cref{eq:CIT}), or $Y(w)$ (which is implied by \cref{eq:CIO}). Violating CIT and CIO would mean that there {\it has to be} such a variable (i.e. dimension) in $X$, that is not represented in the covariate space of $\widetilde{X}$, that has a direct link to either $W$, or $Y(w)$, or both. Such a link is in direct conflict with \cref{eq:CATE:assum:UDM}.

At first glance, CIT and CIO may seem acceptable assumptions. However, we identify two major issues concerning these assumptions: (i) while further specifying the consequences of ignorability despite missingness (cfr. \cref{eq:CATE:assum:UDM}), CIT and CIO are too general, allowing too many different (mostly completely unrealistic) missingness mechanisms without violating \cref{eq:CATE:assum:UDM} nor CIT or CIO; and (ii) they do not allow for treatments to cause missingness, which we have argued above is an important consideration in treatment effects. Furthermore, problem (ii) hints at a larger underlying problem regarding missingness in treatment effects: creating an additional source of bias, which we shall discuss in \Cref{sec:imputation}.

Given the variables in \cref{eq:CATE:assum:UDM} and CIT/CIO ($X$, $\widetilde{X}$, $Z$, $Y(w)$), while keeping
\begin{tikzpicture}[
    roundnode/.style={circle, draw=black,  minimum size=5mm, inner sep=0},
    baseline=-1mm
]
    \node[roundnode] (x) {$X$};
    \node[roundnode] (z) at (.8, 0) {$Z$};
    \node[roundnode, fill=blue!15] (x_tilde) at (1.6, 0) {$\widetilde{X}$};

    \draw[-latex] (x) -- (z);
    \draw[-latex] (z) -- (x_tilde);
    \draw[-latex] (x) to [out=30, in=150] (x_tilde);
\end{tikzpicture} fixed, we can generate 42 DAGs which all respect these assumptions. We have included these DAGs in \Cref{app:CIT_CIO}. Note that each of these 42 DAGs respects the independence statements in \cref{eq:CATE:assum:UDM} and those in either \cref{eq:CIT} or \cref{eq:CIO}. Evaluating these DAGs using the same criteria as we have for MCM leads to only 1 realistic DAG, which turns out to be a special case of MCM. Specifically, when removing the consideration of $\into{Z}$--- treatment causing missingness ---we arrive at the one DAG that is realistic. Without explicit consideration of treatment choices that influence missingness (i.e. splitting $Z$ in $\out{Z}$ and $\into{Z}$), we learn that MCM fits nicely within CIT (and consequentially \cref{eq:CATE:assum:UDM}). Furthermore, this one DAG corresponds exactly to how we introduced M(N)AR in \cref{fig:mar}.

\section{Selective imputation} \label{sec:imputation}

Having introduced MCM, we now explain why missing data in treatment effects should not be brushed over lightly. We discuss why the two common approaches to handle missing data (recall that typically missing data is either imputed, or kept as is) are not equipped to deal appropriately with missing data that interacts with treatment. Here we offer an alternative to these two approaches.

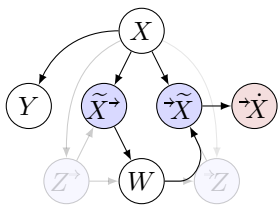
\begin{figure}[h]
    \centering

    \begin{tikzpicture}[
            roundnode/.style={circle, draw=black,  minimum size=6mm, inner sep=0}
        ]
        \node[roundnode] (x) at (0,0) {$X$};
        \node[roundnode, fill=blue!15] (x_tilde_in) at (.5, -1) {$\into{\widetilde{X}}$};
        \node[roundnode, fill=blue!15, opacity=.20] (z_in) at (1, -2) {$\into{Z}$};
        \node[roundnode, fill=blue!15] (x_tilde_out) at (-.5, -1) {$\out{\widetilde{X}}$};
        \node[roundnode, fill=blue!15, opacity=.20] (z_out) at (-1, -2) {$\out{Z}$};
        \node[roundnode] (w) at (0, -2) {$W$};
        \node[roundnode] (y) at (-1.5, -1) {$Y$};
        \node[roundnode, fill=FireBrick!15] (x_hat) at (1.5, -1) {$\into{\dot{X}}$};

        
        \draw[-latex] (x) -- (x_tilde_in);
        \draw[-latex] (x) -- (x_tilde_out);
        
        \draw[-latex] (x_tilde_out) -- (w);
        \draw[-latex, opacity=.20] (z_out) -- (w);
        
        \draw[-latex] (x) to [out=180, in=65] (y);
        \draw[-latex, opacity=.20] (w) -- (z_in);
        \draw[-latex, opacity=.20] (z_in) -- (x_tilde_in);
        \draw[-latex, opacity=.20] (z_out) -- (x_tilde_out);
        \draw[-latex, opacity=.20] (x) to [out=205, in=90] (z_out);
        \draw[-latex, opacity=.1] (x) to [out=335, in=90] (z_in);
        
        \draw[-latex] (w) to [looseness=1, out=0, in=180] (.5,-2) to [out=0, in=299] (.75,-1.5) to [looseness=1, out=115, in=297.5] (x_tilde_in);
        
        \draw[-latex] (x_tilde_in) -- (x_hat);

    \end{tikzpicture}

    \caption{{\bf Selective imputation.} Above DAG depicts MCM where we included $\out{Z}$ and $\into{Z}$ in the factors $\out{\widetilde{X}}$ and $\into{\widetilde{X}}$, respectively;  clearly showing that $\into{\widetilde{X}}$ is a collider. This structure motivates selective imputation: we should only impute $\into{\widetilde{X}}$, where the imputed variables are depicted as $\into{\dot{X}}$.}
    \label{fig:mcm:full}
    
    \rule{\linewidth}{.75pt}
    
\end{figure}

{\bf Why naive approaches do not work.} First, we have to discuss what can go wrong when naively dealing with missing data. To aid our discussion, we simplified MCM in \cref{fig:mcm:full}, merging $\out{Z}$ into $\out{\widetilde{X}}$, and $\into{Z}$ into $\into{\widetilde{X}}$. Note that this is equivalent to \cref{fig:mcm} as the link between the missingness indicator and the observed covariates is completely deterministic. One could, with complete certainty, derive the missingness indicator from the observed covariates. In fact, the only way of obtaining $Z$ is to do exactly that.

Imputing all data, i.e. the first naive approach, has the objective of recovering $X$ from $\widetilde{X}$, as accurately as possible. Given that $Y(w)$, is a direct function of the fully observed set of covariates, $X$, and not the partially observed set, $\widetilde{X}$, regressing $Y(w)$ on $\dot{X}$ should result in better estimates than $\widetilde{X}$. Contrasting supervised learning, however, the treatment effects literature is not {\it only} concerned with better test-set estimates. In particular, the target in treatment effects literature goes beyond the {\it observed} outcomes in the dataset. To accurately estimate treatment effects, one has to account for selection bias \citep{shalit2017estimating}, which is expressed in the path from $X$ to $W$ (see for example, \cref{fig:ignorability}). In our setting, this is more complicated as there exists no direct path. The path from $X$ to $W$ is {\it mediated} by $\out{\widetilde{X}}$. As explained in \Cref{sec:MCM}, treatment is decided on an individual's covariates, but also the missingness in them. As such, imputing away these missing data will result in information loss with respect to selection bias, making it all the more challenging to counteract the bias when predicting a treatment effect.

Should we then impute no data, i.e. the second naive approach? Not imputing data certainly solves the problem of information loss. However, it introduces another problem. We learn from \Cref{sec:MCM}, that in order to estimate treatment effects from data, one has to assume ignorability despite missingness (see \cref{eq:CATE:assum:UDM}). However, from \cref{fig:mcm:full} we learn that $Y(w) \not\independent W | \out{\widetilde{X}}, \into{\widetilde{X}}$, i.e. the potential outcomes are not conditionally independent from treatment. We find the reason in $\into{\widetilde{X}}$, which--- due to being determined by treatment ---composes a collider structure. A property of a collider structure is that it does not bias treatment selection, but when conditioned upon, {\it it does}. If imputing data causes problems, and not imputing data causes problems, what then should we do?

{\bf Selective imputation.} We argue that one {\it should} impute $\into{\widetilde{X}}$, but leave $\out{\widetilde{X}}$ as is. While solutions may vary across datasets, we argue that {\it selective imputation} may offer a general strategy to handle missing data. Consider again \cref{fig:mcm:full}, where we have included $\into{\dot{X}}$ in the MCM-DAG. From evaluating \cref{fig:mcm:full} with this inclusion we learn that,
\begin{equation} \label{eq:selective_imputation}
    Y(0), Y(1) \independent W | \out{\widetilde{X}}, \into{\dot{X}},
\end{equation}
arriving once more at a conditional independence between $W$ and $Y(w)$. Notice that, the independence statement in \cref{eq:CATE:assum:UDM} conditions on the same covariate information (i.e. all is included from $X$), the only thing changed, is that we have to ``forget'' that some data in $\into{\widetilde{X}}$ may be missing as the missingness associated with $\into{\widetilde{X}}$ is determined by the treatment--- a property we need to account for.

{\bf The role of imputation.} From the above discussion, we learn that missingness is what differentiates treatment-subpopulations in $\into{\widetilde{X}}$. The following theorem shows us what exactly imputation needs to accomplish, in order to consider \cref{eq:selective_imputation} to be true (proof of \cref{thm:indep} can be found in \Cref{app:proof}).

\begin{restatable}{theorem}{thm}
\label{thm:indep}
Suppose the graph structure in \cref{fig:mcm:full}, and $W$ and $\into{\dot{X}}$ are independent, then ignorability defined as $Y(0), Y(1) \independent W | \out{\tilde{X}}, \into{\dot{X}}$ holds.
\end{restatable}

From \cref{thm:indep} we learn that, in this context, a good imputation strategy should aim to make $\into{\dot{X}}$ independent of the treatment, such that ignorability may hold. The intuition behind this is that through the act of (proper) imputation we effectively {\it balance} $\into{\tilde{X}}$, making the treatment populations indistinguishable from each other \citep{shalit2017estimating}. The better the imputation is, the closer to $X$ the covariates in $\into{\widetilde{X}}$ will become, reducing influence from $W$. If our imputation is of poor quality, a model may recover information about the originally missing variables, i.e. information on treatment is retained as $\into{\dot{X}}$ still correlates with $W$, which allows bias to creep into our models.

\section{Related work}
We discuss two major perspectives in missingness in a causal setting: {\it $m$-graphs}, and ignorability when there is missingness. A more extensive overview is included in \Cref{app:related_work,app:CIT_CIO}.

{\bf {\em m}-graphs.} An $m$-graph is a graphical model \cite{koller2009probabilistic, pearl2009causality} which explicitly includes missingness indicators as variables \cite{mohan2021graphical}. Naturally, this is related as the way in which we describe MCM is through an $m$-graph. Using an $m$-graph, one determines whether or not an effect is identifiable \cite{pearl2009causality}. While this has received some attention in graphical causality \cite{shpitser2015missing, mohan2014graphical, mohan2013missing} which aims to recover identifiability in MAR and MNAR settings or perform structure learning despite missingness \cite{gain2018structure, kyono2021miracle}, $m$-graphs remain relatively underexplored in a potential outcomes setting, as we have. Note that none of the aforementioned works consider partial imputation, specifically to correctly identify a causal effect. A more in-depth discussion can be found in \Cref{app:related_work}.

{\bf Conditional independence.} \Cref{assum:ignorability} is needed to identify a treatment effect. However, through an $m$-graph, we find that \cref{assum:ignorability} may become unrealistic when there is missingness as it requires the fully observed $X$ which we do not have access to. To our knowledge, \citet{mayer2020doubly} was the first to use $m$-graphs depicting the PO setting (although no mention of $m$-graphs were made in their paper), and we know of none who has done so since. As such, \citet{mayer2020doubly} is by far the most related work to ours, and we compare with them extensively in \Cref{app:CIT_CIO}. \citet{mayer2020doubly} propose two $m$-graphs for CIT and CIO (\cref{eq:CIT,eq:CIO}), included in \cref{fig:CIT:perm:original,fig:CIO:perm:original}, respectively. In \Cref{app:CIT_CIO}, we argue that neither of these $m$-graphs can exist in practice. Furthermore, \cref{fig:CIT:perm:original,fig:CIO:perm:original} only take into account missingness that causes treatment, but not the reverse, making CIT and CIO much more restrictive, as discussed above. 

Recent work by \citet{zhao2021adjust} also concerns treatment effects with missingness. While their method makes strict assumptions (such as randomised treatments), they do make an interesting observation: by specifically violating their assumption that treatment does not cause missingness, there seems to be an additional source of bias \citep[Appendix S1]{zhao2021adjust} (despite random treatment assignment). However, they do not offer a remedy as we do, but only acknowledge that this may happen.

\begin{table*}[t]
    \centering
    
    \caption{{\bf Results.} We evaluate using 10 different train and test sets using 10 {\it different} simulations and averaged the results, std is in brackets. {\bf ATE {\em (upper)}.} For each setting we report an MSE between the predicted ATE and the ground truth. Imputing only $\into{\widetilde{X}}$ (``Selective imputation'') consistently performs best across different ATE prediction methodologies and treatments. {\bf CATE {\em (lower)}.} Imputing only $\into{\widetilde{X}}$, while keeping $\out{\widetilde{X}}$ as is (marked as ``Selective'' below), consistently performs best, in terms of PEHE, across learners and treatments. In all cases, {\bf lower is better}, our proposal is \colorbox{black!10}{shaded}. } \label{tab:results:ate}
    \vspace{5pt}
    \begin{tabularx}{\textwidth}{rll | CCC}
        \toprule
        &\multicolumn{2}{c}{\bf Impute} &{\bf T-Learner} & {\bf Doubly Rob.} & {\bf X-Learner} \\
        &\multicolumn{1}{c}{description} & \multicolumn{1}{c}{covariates} &&&\\
        
        \midrule
        \multicolumn{6}{c}{\it Results on ATE-estimation} \\
        \midrule
        
        \multirow{4}{*}{\rotatebox[origin=c]{90}{\bf MSE}}
        & All & $\out{\widetilde{X}} \cup \into{\widetilde{X}}$ 
        & 0.0951 {\footnotesize(.010)} & 0.0651 {\footnotesize(.008)} & 0.0472 {\footnotesize(.009)}  \\
        
        &Nothing & $\emptyset$ 
        & 0.0642 {\footnotesize(.027)} & 0.0902 {\footnotesize(.018)} & 0.0726 {\footnotesize(.024)} \\
        
        &\CC{10}Selective imputation & \CC{10}$\into{\widetilde{X}}$ 
        &\CC{10} {\bf 0.0403} {\footnotesize(.014)} &\CC{10} {\bf 0.0381} {\footnotesize(.009)} &\CC{10} {\bf 0.0309} {\footnotesize(.014)} \\
        
        &Sel. Complement  &$\out{\widetilde{X}}$
        & 0.0931 {\footnotesize(.026)} & 0.0902 {\footnotesize(.019)} & 0.0984 {\footnotesize(.040)} \\

        \midrule
        \multicolumn{6}{c}{\it Results on CATE-estimation} \\
        \midrule
            
        \multirow{4}{*}{\rotatebox[origin=c]{90}{\bf PEHE}}
        & All & $\out{\widetilde{X}} \cup \into{\widetilde{X}}$ 
        &  0.7603 \footnotesize{(0.051)} & 1.3674 \footnotesize{(1.731)} & 0.6149 \footnotesize{(0.063)}  \\
        &Nothing & $\emptyset$
        & 0.6906 \footnotesize{(0.072)} & 0.9409 \footnotesize{(1.943)}  & 0.3027 \footnotesize{(0.085)} \\
        &\CC{10}Selective imputation & \CC{10}$\into{\widetilde{X}}$ 
        &\CC{10} {\bf 0.4605} \footnotesize{(0.045)} &\CC{10} {\bf 0.2042} \footnotesize{(0.224)} & \CC{10}{\bf 0.2116} \footnotesize{(0.032)} \\
        &Sel. Complement  &$\out{\widetilde{X}}$
        & 0.9158 \footnotesize{(0.064)} & 4.3657 \footnotesize{(8.823)} & 0.4912 \footnotesize{(0.109)} \\

        \midrule
        \multirow{4}{*}{\rotatebox[origin=c]{90}{\bf PEHE\textsubscript{W=0}}}
        & All & $\out{\widetilde{X}} \cup \into{\widetilde{X}}$ 
        & 0.7371 \footnotesize{(0.081)} & 1.2083 \footnotesize{(1.610)}  &  0.6272 \footnotesize{(0.070)} \\
        &Nothing & $\emptyset$ 
        & 0.7015 \footnotesize{(0.100)} &  0.8130 \footnotesize{(1.287)} & 0.2907 \footnotesize{(0.107)}  \\
        &\CC{10}Selective imputation & \CC{10}$\into{\widetilde{X}}$ 
        &\CC{10} {\bf 0.5720} \footnotesize{(0.079)} &\CC{10} {\bf 0.1787} \footnotesize{(0.202)} &\CC{10} {\bf 0.2556} \footnotesize{(0.062)} \\
        &Sel. Complement  &$\out{\widetilde{X}}$
        & 0.9351 \footnotesize{(0.120)} & 4.2306 \footnotesize{(9.056)}  & 0.5198 \footnotesize{(0.156)}\\
        
        \midrule
        \multirow{4}{*}{\rotatebox[origin=c]{90}{\bf PEHE\textsubscript{W=1}}}
        & All & $ \out{\widetilde{X}} \cup \into{\widetilde{X}}$ 
        & 0.7726 \footnotesize{(0.055)} & 1.4419 \footnotesize{(1.802)}  &  0.6097 \footnotesize{(0.068)}\\
        &Nothing & $\emptyset$ 
        & 0.6881 \footnotesize{(0.090)} & 0.9973 \footnotesize{(2.317)}  & 0.3091 \footnotesize{(0.098)} \\
        & \CC{10}Selective imputation & \CC{10}$\into{\widetilde{X}}$ 
        & \CC{10} {\bf 0.4097} \footnotesize{(0.045)} &\CC{10} {\bf 0.2169} \footnotesize{(0.236)}  &\CC{10} {\bf 0.1915} \footnotesize{(0.036)}  \\
        &Sel. Complement  &$\out{\widetilde{X}}$
        & 0.9183 \footnotesize{(0.075)} & 4.4322 \footnotesize{(8.749)} &  0.4803 \footnotesize{(0.130)}\\
        
        \bottomrule
        
    \end{tabularx}
    
\end{table*}
\section{Experiments}

We turn now to empirically validating that selective imputation, i.e. imputing only {\it some} parts of the covariate space, results in better treatment effect predictions.

{\bf Data.} Common in the treatment effects literature is using synthetic datasets. The reason why we have to rely on synthetic data lies at the core of the problem: in any real-world setting, {\it the counterfactual is unobserved}. As such, with synthetic data we can simulate both potential outcomes, effectively observing the counterfactual for evaluation. 

Additionally, simulations allow us to control the missingness mechanism. This allows us to directly test whether treatment effect models are affected by MCM and benefit from Selective Imputation. Our synthetic setup is described in detail in \Cref{app:data}.

For each experiment, we sample 10 different simulated datasets, from which we sample 10 random train and test sets, for each treatment effects method. As our finding holds both for ATE as well as CATE (in finite settings \citep{alaa2018}), we test on both scenarios. Each simulated dataset contains 10k samples. The datasets span 20 dimensions (with factors of equal size), and a missingness rate of 0.3. We have performed multiple ablations on these values in \Cref{app:experiments}. In \Cref{app:experiments} we have also conducted an experiment on the well-known Twins dataset \citep{almond2005costs}, where we manually include some missingness.

{\bf Imputation.} We define four different scenarios: either we impute missing values across (i) all the variables (indicated as ``All''); (ii) none of the variables (``Nothing''); (iii) only the variables in $\into{\widetilde{X}}$ (``Selective imputation''); or (iv) only the variables in $\out{\widetilde{X}}$ (``Sel. Complement''). Following \Cref{sec:imputation}, scenario (iii) (imputing only $\into{\widetilde{X}}$) should yield the best results given that bias is removed, while information with respect to treatment selection is retained. Imputation is performed using MICE \citep{MICE1}.

{\bf Models.} In our experiments, we evaluate the performance of three classes of learners in the imputation scenarios described above: T-learner, Doubly robust (DR) learner, and X-learner. We pair each treatment effects method with \texttt{XGBoost} due to its ability to naturally handle missing values. We refer to \citet{kunzel2019metalearners} or \citet{curth2021nonparametric} for an overview of various learners.

\subsection{Average treatment effects (ATE)} \label{sec:exp:ate}

{\bf Objective.} 
A first estimand to consider is the average treatment effect (defined in \cref{eq:ATE}). While the ATE has been considered since the 1970's \citep{rubin1974estimating}, it is still an important causal estimand today. It can be argued that ATEs suffer more from selection bias than the conditional ATE, as the average is computed over the {\it entire} dataset, unlike the CATE \citep{alaa2018}. In ATE, adjustment plays an important role and is usually achieved by placing non-uniform weights on each element in the dataset when computing an average. These weights often take the form of {\it inverse propensity weights} (IPW), $(p(W=1 | X)^{-1})$, used by, for example, a DR-learner and X-learner.

{\bf Results.} Consider \cref{tab:results:ate} where we reported the mean squared error (MSE) between an estimated ATE and the ground-truth ATE across ten folds, for ten differently sampled simulations (thus spanning 100 trained learners of each type). Given three popular ATE estimation strategies, we find that imputing $\into{\widetilde{X}}$ {\it significantly} performs better across all methods, confirming the insights in \Cref{sec:imputation}.

\subsection{Conditional average treatment effects (CATE)} \label{sec:exp:cate}

{\bf Objective.} 
Similar to \Cref{sec:exp:ate}, we will now evaluate how CATE-learners react to MCM. Specifically, we subject the same treatment effects learners to the same scenarios as we have in \Cref{sec:exp:ate}, and evaluate their CATE-predictions using the PEHE metric described above. As was the case for our ATE experiments, the chosen learners represent a wide range of different methodologies.

CATE estimates are evaluated using the {\it precision in estimating heterogeneous treatment effects} (PEHE), defined as, $\mathbb{E}_\mathcal{X}[(\tau(X) - \hat{\tau}(X))^2]$ in \citet{hill2011bayesian}, where $\hat{\tau}$ is a model's prediction. Naturally, there is a parallel with the MSE we used in \Cref{sec:exp:ate}, where PEHE essentially corresponds to the MSE over the predicted vector of CATEs in a test set.

{\bf Results.} \cref{tab:results:ate} depicts results for various CATE-learners across ten differently sampled simulations, each evaluated with ten different train and test sets--- as we have for our ATE in \Cref{sec:exp:ate}. Given these results, we empirically confirm that one should impute cautiously as imputing all data, no data, or wrong data, consistently performs worse than what we suggest: impute only $\into{\widetilde{X}}$. We confirm our findings with additional configurations in \Cref{app:experiments}.

\vspace{-8pt}
\section{Discussion} \label{sec:discussion}
\vspace{-5pt}
Estimating treatment effects is becoming more important in many practical settings. The adoption of these methods is largely the result of great academic effort to further push the boundaries of these methods' abilities. While practical adoption is indeed a point in favour of causal methods, it comes with a significant downside: unlike in academia, datasets used in practice are often victim to many imperfections. In our paper, we investigated one imperfection in particular: missing variables.

Missing variables in treatment effect settings behave differently from other settings; {\it they should be treated differently as a result}. In our paper, we argue that some missingness can be informative of the treatment, i.e. just like certain characteristics of an individual may determine their treatment, so too can the absence of certain measurements determine treatment. If one should impute these variables, that information is lost and can no longer be used to counter any resulting selection bias.

We argue further that we should not leave {\it all} variables unimputed. When there are variables that are only missing {\it because} an individual was given a particular treatment, then these missing variables are informative of the treatment and result in a covariate shift between treatment groups.

We believe the key finding of our paper is summarised as follows: {\it more care and thought should be put into imputing missing data when estimating treatment effects}. We confirm intuitively, theoretically, and empirically that selectively imputing missing variables can improve our ability to estimate treatment effect, while imputing the wrong variables {\it can lead to poorly modelled treatment effects}.

\bibliography{bibliography}
\bibliographystyle{apalike}

\clearpage
\onecolumn
\appendix

\part{Appendix: MCM}
\parttoc

\section{Proof of \Cref{thm:indep}}\label{app:proof}

\thm*

\begin{proof}
Consider $p(Y=y, W=w | \out{\widetilde{X}} = \tilde{x}, \into{\dot{X}} = \dot{x})$, which we write as $p(y, w | \tilde{x}, \dot{x})$.
Then by introducing $X=x$ and using the definition of the conditional density, we have
\begin{equation} \label{eq:int}
\begin{aligned}
    p(y, w | \tilde{x}, \dot{x}) &= \int_x p(y, w, x| \tilde{x}, \dot{x}) \\
    &= \int_x \frac{p(y, w, x, \tilde{x}, \dot{x})}{p(\tilde{x}, \dot{x})}
\end{aligned}
\end{equation}
We can rewrite $p(y, w, x, \tilde{x}, \dot{x})$ using the graph structure as
\begin{equation} \label{eq:graphcon}
    p(y, w, x, \tilde{x}, \dot{x}) = p(y | x)p(\tilde{x}|x)p(w|\tilde{x})p(\dot{x}|x, w)p(x)\,.
\end{equation}
Noting that $X$ is {\em not} a collider of $W$ and $\into{\dot{X}}$, and by independence of $W$ and $\into{\dot{X}}$ we have that 
\begin{equation} \label{eq:indep}
    p(\dot{x} | x, w) = p(\dot{x} | x)
\end{equation}
By \Cref{app:lemma}, $\out{\widetilde{X}}$ and $\into{\dot{X}}$ are conditionally independent given $X$. Therefore $p(\tilde{x} | x) p(\dot{x} | x) = p(\tilde{x}, \dot{x} | x)$. Similarly, $X$ d-separates $Y$ and the pair $(\out{\widetilde{X}}, \into{\dot{X}})$ and thus, 
\begin{equation} \label{eq:joint}
    p(y|x)p(\tilde{x}|x)p(\dot{x}|x) = p(y, \tilde{x}, \dot{x} | x).
\end{equation}
Substituting (\ref{eq:indep}) and (\ref{eq:joint}) into (\ref{eq:graphcon}), gives us
\begin{equation}
    p(y, w, x, \tilde{x}, \dot{x}) = p(y, \tilde{x}, \dot{x} | x)p(w | \tilde{x})p(x).
\end{equation}
We then note that $\out{\widetilde{X}}$ is {\em not} a collider of $W$ and $\into{\dot{X}}$ so that,
\begin{equation}
    p(w | \tilde{x}) = p(w | \tilde{x}, \dot{x}).
\end{equation}
Substituting this all into (\ref{eq:int}) gives us
\begin{equation}
\begin{aligned}
    p(y, w | \tilde{x}, \dot{x}) &= \int_x \frac{p(y, \tilde{x}, \dot{x} | x)p(w | \tilde{x}, \dot{x})p(x)}{p(\tilde{x}, \dot{x})} \\
    &= \frac{p(y, \tilde{x}, \dot{x})p(w | \tilde{x}, \dot{x})}{p(\tilde{x}, \dot{x})} \\
    &= p(y | \tilde{x}, \dot{x})p(w | \tilde{x}, \dot{x})
\end{aligned}
\end{equation}
which proves the result.

\end{proof}

\begin{lemma} \label{app:lemma}
Suppose the setting of \Cref{thm:indep}. Then $\out{\tilde{X}}$ and $\dot{X}$ are conditionally independent given $X$.
\end{lemma}

\begin{proof}
By appealing to the DAG to write the full joint distribution of $p(x, \out{\tilde{x}}, w, \into{\tilde{x}}, \dot{x})$ we have
\begin{equation}
\begin{aligned}
    p(\out{\tilde{x}}, \dot{x} | x) &= \int_{\into{\tilde{x}}, w} p(\out{\tilde{x}} | x)p(w|\out{\tilde{x}}, x) p(\into{\tilde{x}} | x, w) p(\dot{x} | \into{\tilde{x}}, x) \\
    &= \int_w p(\out{\tilde{x}} | x) p(w | \out{\tilde{x}})p(\dot{x} | x) \\
    &= p(\out{\tilde{x}} | x) p(\dot{x} | x) \int_w p(w | \out{\tilde{x}}) \\
    &= p(\out{\tilde{x}} | x) p(\dot{x} | x)
\end{aligned}
\end{equation}
\end{proof}

\section{Related work} \label{app:related_work}

Ideas of dealing with missing variables have been around for some time \cite{rosenbaum1984reducing}, leading to many relevant works. However, let us first address works that may {\it seem} relevant, but are not. In particular, we are not concerned with missingness in either treatment or outcome variables \citep{robins2000, rotnitzky1998semiparametric}, only in the confounding variables $X$.

{\bf Censoring.} The case of missing outcomes, or non-response \cite{rotnitzky1998semiparametric}, is well known in survival analysis \citep[Chapter 17]{van2003unified, hernan2002estimating, hernan2010causal}. While survival analysis is indeed concerned with a particular case of missingness (that of outcomes), it is only relevant in a setting over time\footnote{In fact, survival analysis is sometimes termed ``time-to-event analysis'' \cite{tolles2019}.}. For example, if one wishes to predict the number of years a patient will live, when given a drug, we will only know the answer when they die. If the patient hasn't died yet at the time of analysis, we actually don't know the answer and have to consider the outcome missing (censored). Some works that aim to connect both fields include \citet{stitelman2012general, curth2021survite, tabib2020non, andersen2017causal, prague2016accounting}.

{\bf {\em m}-graphs.} Another line of research is that of {\it missingness graphs} or $m$-graphs \cite{mohan2021graphical}. In short, an $m$-graph is a probabilistic graphical model \cite{koller2009probabilistic, pearl2009causality} which explicitly includes missingness indicators as variables. Naturally, this is related as the way in which we describe MCM is through an $m$-graph. Using an $m$-graph, one can determine whether or not an effect is identifiable \cite{pearl2009causality}. While this has received some attention in graphical causality \cite{shpitser2015missing, mohan2014graphical, mohan2013missing} which aim to recover identifiability in MAR and MNAR settings or perform structure learning despite missingness \cite{gain2018structure, kyono2021miracle}, never before were these $m$-graphs considered in a potential outcomes setting, as we have. Note that none of the aforementioned works consider partial imputation, specifically to correctly identify a causal effect.

{\bf Conditional independence.} Identifying the causal effect in a potential outcomes setting, requires conditional independence: $W \independent Y(w) |X$, for all $w$ and $X$. If this is not the case, then we cannot interpret the estimated effect as {\it causal}. While much of the potential outcomes literature considers the setting where this is simply the case, it seems that these assumptions are simply carried over to the case that includes missingness. Having only three variables, the conditional independence assumptions merely amount to assuming that variables beyond $X$ are {\it ignorable}, i.e. all relevant variables are measured. While this may very well be the case for three variables, this assumption may be ``stretching it'' when we include missingness indicators. This is especially clear when we portray the situation as an $m$-graph.

We are not the first to portray missingness in the potential outcomes setting as an $m$-graph. To our knowledge, \citet{mayer2020doubly} was the first to do so (although no mention of $m$-graphs was made in their paper), and we know of none who has done so since. As such, \citet{mayer2020doubly} is by far the most related work to ours, and we compare with them extensively in \Cref{app:CIT_CIO}. In particular, \citet{mayer2020doubly} proposes two $m$-graphs for CIT and CIO (\cref{eq:CIT,eq:CIO}), which we have included in \cref{fig:CIT:perm:original,fig:CIO:perm:original}, respectively. In \Cref{app:CIT_CIO}, we argue that neither of these $m$-graphs can exist in practice. Furthermore \cref{fig:CIT:perm:original,fig:CIO:perm:original}, only take into account missingness that causes treatment, but not the reverse, making CIT and CIO much more restrictive, as we have indicated in our main text.

\section{Additional experiments} \label{app:experiments}

Here we present the following additional experiments to those presented in our main text: the same experimental setup as in \cref{tab:results:ate} with an alternative imputation scheme, a sensitivity analysis on the amount of missingness in a dataset, and the same experimental setup as in \cref{tab:results:ate}, but using a different (real) dataset.

\begin{figure*}[t]
    \centering

    \resizebox{\textwidth}{!}{\begin{tikzpicture}
        \pgfplotsset{footnotesize,samples=10}
        \begin{groupplot}[group style = {group size = 2 by 1, horizontal sep = 30pt}, width = 6.0cm, height = 4.0cm]
            
            \nextgroupplot[ 
                title = {T-Learner},
                legend style = { 
                    column sep = 5pt, 
                    legend columns = -1, 
                    legend to name = grouplegend, anchor=west
                    },
                xlabel={amount of missingness},
                ylabel={PEHE},
                ymin=0.2, ymax=1.4,
                xtick={0.1, 0.2, 0.3, 0.4, 0.5},
                ytick={0.4,0.6,0.8,1,1.2},
                ymajorgrids=true,
                grid style=dashed,
            ]
                \addplot[color=Black,mark=square,]
                    coordinates {(0.1, 0.90667) (0.2, 1.08401) (0.3, 1.04266) (0.4, 0.79051) (0.5, 0.58199)};
                    
                \addplot[color=Blue,mark=triangle,]
                    coordinates {(0.1, 0.87106) (0.2, 0.87011) (0.3, 0.96961) (0.4, 0.77991) (0.5, 0.83227)};
                    
                \addplot[color=FireBrick,mark=o,]
                    coordinates {(0.1, 0.78409) (0.2, 0.53562) (0.3, 0.52142) (0.4, 0.49251) (0.5, 0.5193)};
                    
                \addplot[color=DarkGreen,mark=diamond,]
                    coordinates {(0.1, 0.91584) (0.2, 1.18526) (0.3, 1.21693) (0.4, 0.96973) (0.5, 0.79995)};

            \nextgroupplot[
                title = {X-Learner},
                legend style = { column sep = 5pt, legend columns = 1, legend to name = grouplegend,},
                xlabel={amount of missingness},
                xmin=0.05, xmax=0.55,
                ymin=0, ymax=1,
                xtick={0.1, 0.2, 0.3, 0.4, 0.5},
                ytick={0.2,0.4,0.6,0.8,1,1.2,1.4,1.6,1.8,2},
                ymajorgrids=true,
                grid style=dashed,
            ]
                \addplot[color=Black,mark=square,]
                    coordinates {(0.1, 0.4602) (0.2, 0.65027) (0.3, 0.62072) (0.4, 0.56482) (0.5, 0.44103)};
                    \addlegendentry[Black]{All}
                    
                \addplot[color=Blue,mark=triangle,]
                    coordinates {(0.1, 0.15049) (0.2, 0.22665) (0.3, 0.25882) (0.4, 0.28314) (0.5, 0.21778)};
                    \addlegendentry[Black]{Nothing}
                    
                \addplot[color=FireBrick,mark=o,]
                    coordinates {(0.1, 0.12644) (0.2, 0.17995) (0.3, 0.1631) (0.4, 0.20949) (0.5, 0.20363)};
                    \addlegendentry[Black]{Selective}
                    
                \addplot[color=DarkGreen,mark=diamond,]
                    coordinates {(0.1, 0.46828) (0.2, 0.43843) (0.3, 0.46647) (0.4, 0.32621) (0.5, 0.26493)};
                    \addlegendentry[Black]{Sel. Complement}
        \end{groupplot}
        \node at ($(group c2r1) + (4.5cm,0cm)$) {\ref{grouplegend}}; 
    \end{tikzpicture}}
    \caption{{\bf Sensitivity analysis on the amount of missingness: CATE.} We report for two learners their performance (PEHE, y-axis) as a function of the amount of missingness (x-axis). For each setting of the amount of missingness, we sample 50 different train and test sets to calculate an average and std (in line with what is reported in \cref{tab:results:ate} and documented in our provided code-files).}
    \label{fig:app:res:amount_miss}
    \rule{\linewidth}{.75pt}
\end{figure*}
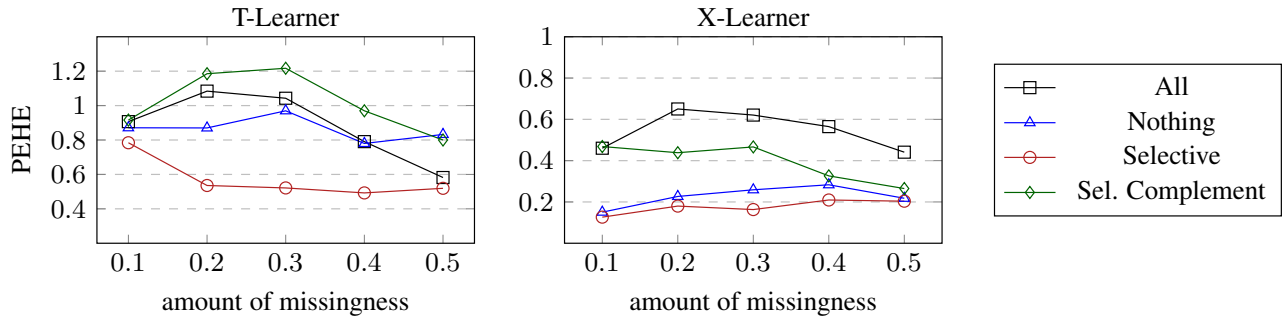

{\bf Imputation with representation learning.} Consider \cref{tab:res:gain}, where we have repeated the experiment in \cref{tab:results:ate} with GAIN \citep{GAIN}; a deep learning imputation method based on GANs \citep{goodfellow2014generative}. While notoriously difficult to train, we find the combination of GAIN and MCM to be an interesting one. Specifically, the discriminator is tasked with classifying imputed samples from real samples, making it harder to predict which samples were imputed. The way the discriminator does this, is by predicting a binary mask of sorts, corresponding to what we have defined as $Z$. Having 50/50 prediction (resulting in the worst cross-entropy for binary prediction), we say that $Z$ is independent of the imputed result, which then corresponds with \cref{thm:indep}. Note that in \cref{thm:indep} we pose independence of $W$. Seeing that $\into{Z}$ is a mediator variable between $W$ and $\into{\widetilde{X}}$, we have $\into{\widetilde{X}} \independent W | X$ given that $\into{\widetilde{X}} \independent Z | X$.

{\bf Sensitivity.} As a second additional experiment, consider \cref{fig:app:res:amount_miss}, where we have repeated the same experiment in \cref{tab:results:ate}, but with different levels of missingness. Not only do these results confirm that selective imputation is best across the board, they also bring to light an interesting phenomenon: the trade-off between prediction gain from imputation versus bias reduction. As $Y(w)$ is a function of the fully observed $X$, it makes sense that predicting $Y(w)$ from $X$ directly would yield the best results. This insight is the motivation why imputation is so common in standard supervised learning. From our discussion in \Cref{sec:imputation}, we learn that imputation in the treatment effects setting introduces bias, but perhaps there is only very little bias. In some cases, it might be better to impute, {\it despite bias}.

What we see in \cref{fig:app:res:amount_miss} is the combination of the balance trade-off described above, with another realisation: if too much data is missing, bias might actually reduce, as the associated patterns in the data become less apparent. Take the most extreme case where all data in $\into{\widetilde{X}}$ is missing. Of course, then there is no bias anymore, as there is no more data to bias. With less bias in the data, it might be best for a model to impute. This is clearly visible in the leftmost plot (T-Learner), where not imputing does not seem to improve with more missingness; only the ``Sel. Complement'' and ``All''--- which introduce bias ---seem to improve as more data is missing, likely since there is less bias available to introduce by imputing. Overall, we find that ``Selective'' outperforms the rest.  

{\bf Semi-real data.} In order to not only rely on synthetic data, we also test our findings in the well-known and used Twins dataset \cite{almond2005costs}. Naturally, it is semi-real, as we need to simulate the missingness mechanism such that we know the exact factors of $\into{Z}$ and $\out{Z}$. In order to not let any other bias creep into the missingness mechanisms, we removed all rows that already contained missing variables. 

Consider \cref{tab:res:twins} which presents results using the same setup as \cref{tab:results:ate,tab:res:gain}. Again, we find no surprises and report consistent performance improvements when selectively imputing the $\into{Z}$ factor of the data. For this, we again relied on MICE for imputing the data \cite{MICE1}.

\begin{table*}[t]
    \centering
    
    \caption{{\bf CATE results with GAIN.} We repeat our experiment from the main text. For each setting, we report a PEHE on 100 {\it differently sampled} simulations and averaged the results, standard deviation is reported in brackets. Imputing only $\into{\widetilde{X}}$, while keeping $\out{\widetilde{X}}$ as is, performs best across learners and treatments. {\bf Lower is better}, our proposal is \colorbox{black!5}{shaded}.} \label{tab:res:gain}
    \begin{tabularx}{\textwidth}{rll | CCC}
        \toprule
        &\multicolumn{2}{c}{\bf Impute} &{\bf T-learner} & {\bf Doubly rob.} & {\bf X-learner} \\
        &\multicolumn{1}{c}{description} & \multicolumn{1}{c}{covariates} &&&\\
        
        \toprule
            
        \multirow{4}{*}{\rotatebox[origin=c]{90}{\bf PEHE}}
        & All & $ \out{\widetilde{X}} \cup \into{\widetilde{X}} $ 
        &  0.8720  \footnotesize{(0.072)} & 1.4276  \footnotesize{(1.410)} & 0.5388  \footnotesize{(0.112)}  \\
        &Nothing & $\emptyset$
        &  0.8233  \footnotesize{(0.064)} & 0.2421  \footnotesize{(0.579)} & 0.4724 \footnotesize{(0.072)} \\
        &\CC{5}Selective & \CC{5}$\into{\widetilde{X}}$ 
        & \CC{5} {\bf 0.5260}  \footnotesize{(0.105)} & \CC{5}{\bf 0.1497}  \footnotesize{(0.210)} & \CC{5} {\bf 0.2633 } \footnotesize{(0.069)} \\
        &Sel. Complement  &$\out{\widetilde{X}}$
        &  0.8855  \footnotesize{(0.056)} & 0.4726  \footnotesize{(0.971)} & 0.6924  \footnotesize{(0.152)} \\

        \midrule
        \multirow{4}{*}{\rotatebox[origin=c]{90}{\bf PEHE\textsubscript{W=0}}}
        & All & $\out{\widetilde{X}} \cup \into{\widetilde{X}}$ 
        &  0.8918  \footnotesize{(0.107)} &1.6743  \footnotesize{(1.218)} &  0.5684  \footnotesize{(0.138)} \\
        &Nothing & $\emptyset$ 
        &  0.7079  \footnotesize{(0.087)} & 0.2719  \footnotesize{(0.340)} & 0.4829  \footnotesize{(0.116)}  \\
        &\CC{5}Selective & \CC{5}$\into{\widetilde{X}}$ 
        & \CC{5}  {\bf 0.5754}  \footnotesize{(0.092)} & \CC{5}{\bf 0.1539}  \footnotesize{(0.209)} &\CC{5} {\bf 0.2699 } \footnotesize{(0.064)} \\
        &Sel. Complement  &$\out{\widetilde{X}}$
        &  0.8342  \footnotesize{(0.105)} & 0.3531  \footnotesize{(0.604)} & 0.7313  \footnotesize{(0.200)}\\
        
        \midrule
        \multirow{4}{*}{\rotatebox[origin=c]{90}{\bf PEHE\textsubscript{W=1}}}
        & All & $ \out{\widetilde{X}} \cup \into{\widetilde{X}}$ 
        &  0.8650  \footnotesize{(0.074)} & 2.0708  \footnotesize{(0.709)}&  0.5169  \footnotesize{(0.107)}\\
        &Nothing & $\emptyset$ 
        & 0.8844  \footnotesize{(0.080)} & 0.2798  \footnotesize{(0.706)} & 0.4650  \footnotesize{(0.070)} \\
        & \CC{5}Selective & \CC{5}$\into{\widetilde{X}}$ 
        & \CC{5}  {\bf 0.5035}  \footnotesize{(0.132)} & \CC{5} {\bf 0.1482}  \footnotesize{(0.213)} &\CC{5}{\bf  0.2590}  \footnotesize{(0.075)}  \\
        &Sel. Complement  &$\out{\widetilde{X}}$
        &  0.9149  \footnotesize{(0.068)} & 0.5332  \footnotesize{(1.197)}&  0.6664  \footnotesize{(0.157)}\\
        
        \bottomrule
        
    \end{tabularx}
    
\end{table*}

\begin{table*}[h!]
    \centering
    
    \caption{{\bf Results on Twins data.} We repeat our experiments in \cref{tab:results:ate}, presented in our main text, but on a different (real world) dataset \cite{almond2005costs}. Again, we sampled 100 different subsets of the data, with a different outcome and missing variables, and averaged the results with standard deviation in brackets. Different from our reported results in the main text, we normalised the results in order to better compare across methods as we found the performance discrepancy to increase when compared to the fully synthetic setup. Again, {\bf lower is better}, and our proposal is \colorbox{black!5}{shaded}.} \label{tab:res:twins}
    \begin{tabularx}{\textwidth}{rll | CCC}
        \toprule
        &\multicolumn{2}{c}{\bf Impute} &{\bf T-learner} & {\bf Doubly rob.} & {\bf X-learner} \\
        &\multicolumn{1}{c}{description} & \multicolumn{1}{c}{covariates} &&&\\
        
        \midrule
        \multicolumn{6}{c}{\it Results on ATE-estimation} \\
        \midrule
        
        \multirow{4}{*}{\rotatebox[origin=c]{90}{\bf MSE}}
        & All & $ \out{\widetilde{X}} \cup \into{\widetilde{X}} $ 
        &  5.7554  \footnotesize{(2.535)} &  8.9752  \footnotesize{(3.047)} &  30.792  \footnotesize{(8.459)} \\
        
        &Nothing & $\emptyset$ 
        &  6.3382  \footnotesize{(3.156)} &  10.310  \footnotesize{(4.182)} &  8.9802  \footnotesize{(5.661)} \\
        
        &\CC{5}Selective & \CC{5}$\into{\widetilde{X}}$ 
        & \CC{5} {\bf 3.8482}  \footnotesize{(2.017)}& \CC{5} {\bf 5.7543}  \footnotesize{(2.403)}& \CC{5} {\bf 8.1534}  \footnotesize{(4.562)} \\
        
        &Sel. Complement  &$\out{\widetilde{X}}$
        &  6.5303  \footnotesize{(3.150)} &  9.2979  \footnotesize{(3.739)} &  10.369  \footnotesize{(6.097)} \\

        \midrule
        \multicolumn{6}{c}{\it Results on CATE-estimation} \\
        \midrule

        \multirow{4}{*}{\rotatebox[origin=c]{90}{\bf n-PEHE}}
        & All & $ \out{\widetilde{X}} \cup \into{\widetilde{X}}$ 
        &  0.7572  \footnotesize{(0.101)} &  0.0699  \footnotesize{(0.122)} &  0.8971  \footnotesize{(0.065)} \\
        &Nothing & $\emptyset$
        &  0.7393  \footnotesize{(0.295)} &  0.0015  \footnotesize{(0.003)} &  0.9191  \footnotesize{(0.114)} \\
        &\CC{5}Selective & \CC{5}$\into{\widetilde{X}}$ 
        & \CC{5} {\bf 0.7289}  \footnotesize{(0.302)}& \CC{5} {\bf 0.0013}  \footnotesize{(0.015)}& \CC{5} {\bf 0.8868}  \footnotesize{(0.119)} \\
        &Sel. Complement  &$\out{\widetilde{X}}$
        &  0.7591  \footnotesize{(0.095)} &  0.1441  \footnotesize{(0.376)} &  0.9067  \footnotesize{(0.062)} \\

        \midrule
        \multirow{4}{*}{\rotatebox[origin=c]{90}{\bf n-PEHE\textsubscript{W=0}}}
        & All & $ \out{\widetilde{X}} \cup \into{\widetilde{X}} $ 
        &  0.9590  \footnotesize{(0.064)} &  0.0792  \footnotesize{(0.116)} &  0.9191  \footnotesize{(0.071)} \\
        &Nothing & $\emptyset$ 
        &  0.7451  \footnotesize{(0.079)} &  0.0026  \footnotesize{(0.004)} &  0.7557  \footnotesize{(0.066)} \\
        &\CC{5}Selective & \CC{5}$\into{\widetilde{X}}$ 
        & \CC{5} {\bf 0.7407}  \footnotesize{(0.076)}& \CC{5} {\bf 0.0028}  \footnotesize{(0.028)}& \CC{5} {\bf 0.7406}  \footnotesize{(0.064)} \\
        &Sel. Complement  &$\out{\widetilde{X}}$
        &  0.9570  \footnotesize{(0.066)} &  0.1522  \footnotesize{(0.376)} & 0.9139  \footnotesize{(0.068)} \\
        
        \midrule
        \multirow{4}{*}{\rotatebox[origin=c]{90}{\bf n-PEHE\textsubscript{W=1}}}
        & All & $ \out{\widetilde{X}} \cup \into{\widetilde{X}} $ 
        &  0.7152  \footnotesize{(0.109)} &  0.0695  \footnotesize{(0.122)} &  0.8803  \footnotesize{(0.066)} \\
        &Nothing & $\emptyset$ 
        &  0.7188  \footnotesize{(0.319)} &  0.0015  \footnotesize{(0.002)} &  0.9090  \footnotesize{(0.120)} \\
        & \CC{5}Selective & \CC{5}$\into{\widetilde{X}}$ 
        & \CC{5} {\bf 0.7086}  \footnotesize{(0.327)}& \CC{5} {\bf 0.0001}  \footnotesize{(0.015)}& \CC{5} {\bf 0.8700}  \footnotesize{(0.126)} \\
        &Sel. Complement  &$\out{\widetilde{X}}$
        &  0.7175  \footnotesize{(0.104)} &  0.1438  \footnotesize{(0.363)} &  0.8810  \footnotesize{(0.064)} \\
        
        \bottomrule
        
    \end{tabularx}
    
\end{table*}

\section{Data and reproducibility} \label{app:data}

Standard in treatment effects literature, proper empirical validation requires synthetic data. Here we describe our synthetic setup, and how we sample from it precisely. We do this by including Python code, here in the appendix. For further clarifications (beyond this appendix), both on data generation as well as selective imputation, we refer the reader to our online code repository:

\begin{mdframed}[backgroundcolor=blue!5]
\centering
\vspace{5pt}
\url{https://github.com/jeroenbe/mcm}
\vspace{5pt}
\end{mdframed}

We also like to point the interested reader to a wide collection of related code repositories to this work, which includes the code repository of this paper:

\begin{mdframed}[backgroundcolor=blue!5]
\centering
\vspace{5pt}
\url{https://github.com/vanderschaarlab/mlforhealthlabpub}
\vspace{5pt}
\end{mdframed}

{\bf [Step 1] \emph{Sample $X$}.} In order for imputation to make sense, there has to exist some correlation between variables in $X$, as such sampling from a standard normal with the identity matrix as a covariance matrix, will not suffice. As such, we sample from a normal distribution with a random (positive semidefinite) covariance matrix, spanning 20 dimensions:

\begin{lstlisting}[language=Python]
import numpy as np
    
def _generate_covariates(d, n):
    assert 0 < d
    assert 0 < n

    A = np.random.rand(d,d)
    cov = np.dot(A, A.transpose())

    X = np.random.multivariate_normal(np.zeros(d), cov, size=n)
    X /= (X.max() - X.min())

    return X
\end{lstlisting}
  
In the above code, we have \texttt{d=20} the dimension count, and \texttt{n=10000} the sample size. In our experiments, we sample 10 different train and test sets from the data to calculate  per simulation descriptives.

{\bf [Step 2] \emph{Sample $\out{Z}$}.} From $X$ we generate $\out{Z}$ as follows:
\begin{lstlisting}[language=Python, firstnumber=14]
    highest_border = X[:,:z_dim].argsort(axis=1)[:,
        -int(np.max((int(np.round(amount_of_missingness * z_dim)), 1)))]
    Z_out = list(x >= x[highest_border[i]] for i, x in enumerate(X[:,:z_dim]))
    Z_out = np.array(Z_out).astype(int)
    Z_out = np.abs(Z_out-1)
\end{lstlisting}
 
In the above code, we have two main parameters: \texttt{z\_dim=10} indicating the number of variables in $\out{Z}$, and \texttt{amount\_of\_missingness=0.3} indicating the fraction of the data that is missing. in \texttt{ln 14} above, we calculate a threshold, when the value of $X_i$ above this threshold, that variable is missing. This threshold corresponds to the number of dimensions that need to be missing in order to respect \texttt{amount\_of\_missingness}. Notice that $\out{Z}$ now corresponds to \texttt{X[:,:z\_dim]}.  

{\bf [Step 3] \emph{Sample $W$}.} We make very explicit that treatment choices depend on $\out{Z}$.
\begin{lstlisting}[language=Python, firstnumber=19]
    W = []
    for z_d in Z_out:
        if 0 == z_d[-1]:
            w = 0
        elif 0 in z_d[:int(np.floor(z_dim/2))]:
            w = 1
        else:
            w = np.random.binomial(1, .5)
        W.append(w)
    W = np.array(W)
\end{lstlisting}

{\bf [Step 4] \emph{Sample $\into{Z}$}.} Sampling $\into{Z}$, which depends on $W$ in such a way that the arrow is identifiable, requires interaction between $X$ and $W$\footnote{Note that in the non-binary setting, this interaction may not be necessary.}. For this, we sample two random vectors (\texttt{theta\_z\_in\_0}, and \texttt{theta\_z\_in\_1}) and let those interact with $X$ to decide $\into{Z}$. As was the case with $\out{Z}$, we also calculate the number of variables that should be affected (calculated in \texttt{ln 29-32}). Note that the \texttt{dim\_count} in this setting only corresponds {\it approximately} to the \texttt{amount\_of\_missingness}.  
\begin{lstlisting}[language=Python, firstnumber=29]
    import scipy

    dim_count = np.round(amount_of_missingness * (d - z_dim) * 2)
    dim_count = np.max((dim_count, 1))
    dim_count = np.min((dim_count, int((d - z_dim) / 2)))
    dim_count = int(dim_count)
    
    theta_z_in_0 = np.full(dim_count, scipy.stats.norm.ppf(1 - amount_of_missingness))
    theta_z_in_1 = np.full(dim_count, scipy.stats.norm.ppf(1 - amount_of_missingness))

    Z_in = np.zeros((n, d - z_dim))
    for i, z in enumerate(Z_in):
        x = X[i, z_dim:z_dim+dim_count]
        if W[i]:
            Z_in[i, -dim_count:] = (x - X[:, z_dim:z_dim+dim_count].mean(axis=0)) > (theta_z_in_1 * x.std(axis=0))
        else:
            Z_in[i, :dim_count] = (x - X[:, z_dim:z_dim+dim_count].mean(axis=0)) > (theta_z_in_0 * x.std(axis=0))
    Z_in = np.abs(Z_in-1)
\end{lstlisting}

{\bf [Step 5] Sample $Y(w)$.} Generating outcomes is done simply by sampling two linear functions.
\begin{lstlisting}[language=Python, firstnumber=47]
    def _generate_outcomes(X, W):
        theta = np.random.randn(X.shape[1]) / 10
        theta_y0 = np.ones(X.shape[1]) + theta
        theta_y1 = np.ones(X.shape[1]) * -1 + theta
        
        Y0 = np.sum(X * theta_y0, 1)
        Y1 = np.sum(X * theta_y1, 1)
    
        Y = np.array([Y0[i] if w == 0 else Y1[i] for i, w in enumerate(W)]) 
        Y += np.random.randn(X.shape[0])*.1
    
        CATE = Y1 - Y0
        
        return Y0, Y1, CATE, Y
\end{lstlisting}

{\bf [Step 6] Identifiability.} All the above are simple functions, which are made identifiable (such that they respect the DAG) through a non-linearity. Then, $\widetilde{X}$ is generated by combining $Z$ into $X$.  
\begin{lstlisting}[language=Python, firstnumber=61]
    X = np.abs(X)

    X_tilde = X.copy()
    X_tilde[:,z_dim:][Z_out==0] = missing_value
    X_tilde[:,:z_dim][Z_in==0] = missing_value
\end{lstlisting}

\section{MCM examples} \label{app:example}

In the main manuscript, we provide an example (knee pain) of how MCM arises in practice, and in particular how missingness may be caused by the treatment.
In addition to our example, we provide an additional (non-medical) example below.

\textbf{Job-training program.}
We now present a quite different example of how MCM may arise in practice from outside of medicine. 
\begin{adjustwidth}{.5cm}{}
{\it Imagine a job-training program to boost employment. Before sponsoring such a program, a legislative body may want to estimate its effect before widespread adoption. Using past data on the program, the body has to rely on causal methods to infer the effect of the training program before formal adoption.}
\end{adjustwidth}

Naturally, one requires data to learn a treatment effects model. But as is motivated throughout the literature, these data may be biased and cannot directly be used to learn a good treatment effects model.
\begin{adjustwidth}{.5cm}{}
{\it In the past, the program was not offered to everyone. Those with a job were less likely to be considered. As were those that had been unemployed for a long time. Perhaps age played a role; being less likely to switch careers, older people would benefit less from job training.}
\end{adjustwidth}

While there exists ``natural'' bias, for example age may influence whether someone may or may not receive the treatment. In MCM, we argue that the fact that a variable is missing, may further bias the data.
\begin{adjustwidth}{.5cm}{}
{\it Without knowledge of someone's age, it was deemed better to not offer the program. This is understandable, as the high price of the program requires some level of certainty.}
\end{adjustwidth}

Bias does not stop there. Even when treatment is assigned, data may be biased further.
\begin{adjustwidth}{.5cm}{}
{\it Once an applicant has accepted their offer for job training, the organisers require some additional information. Imagine, the program's registration process asking the applicant's current address, the job of their spouse, the number of children they have, etc. Had they not accepted to participate in the program, they would be less likely to provide this additional information.}
\end{adjustwidth}

Understanding where bias may come from, what could go wrong when these sources are ignored, or brushed over too lightly? One way this bias may sneak into our models is by using an estimator that knows how to handle missing data (such as an xgboost model) and choose to not impute anything.
\begin{adjustwidth}{.5cm}{}
{\it The organisers found that they had vastly overestimated the effect of job training. Those willing to provide additional information were also very motivated to advance their careers, and conversely, those that were reluctant to provide information were not motivated at all. By not imputing, the organisers' model was identifying the effect of motivation, not of job training.}
\end{adjustwidth}

However, when naively imputing everything, bias still creeps into our model.
\begin{adjustwidth}{.5cm}{}
{\it However, when they did decide to impute, their model seemed to underestimate the treatment effect, as the imputation strategy filled in absent age close to the mean age, however, it was mostly older people who did not provide their age. When adjusting for ``younger'' people, their model was actually confusing its estimates with data from older people, where their job-training program yielded less effect.}
\end{adjustwidth}
This final point is subtle. Through the act of debiasing, the model ignored vital information. Indeed, there exists some bias from $X$ to $W$, through the missing variables: namely, older people are more likely to have missing age. This is a clear example of removing a placeholder (the missingness indicator) for information that contributed to bias. Crucially, age plays an important role in outcome as well, propagating this bias into the effect estimates.

\section{Exhaustive DAG-search over CIT and CIO} \label{app:CIT_CIO}
CIT and CIO are assumptions over $X$, $\widetilde{X}$, $Z$, $Y(w)$, and $W$. While keeping 
\begin{tikzpicture}[
    roundnode/.style={circle, draw=black,  minimum size=5mm, inner sep=0},
    baseline=-1mm
]
    \node[roundnode] (x) {$X$};
    \node[roundnode] (z) at (1, 0) {$Z$};
    \node[roundnode, fill=blue!15] (x_tilde) at (2, 0) {$\widetilde{X}$};

    \draw[-latex] (x) -- (z);
    \draw[-latex] (z) -- (x_tilde);
    \draw[-latex] (x) to [looseness=.9, out=45, in=180] (1,.5) to [looseness=.9, in=135, out=0] (x_tilde);
\end{tikzpicture} fixed, we discuss each other DAG that respects \cref{eq:CIT} or \cref{eq:CIO}. 

In particular, for a DAG to be a valid description for missingness in CATE, they {\it should contain}:\\
\begin{tikzpicture}[
    roundnode/.style={circle, draw=black,  minimum size=5mm, inner sep=0},
    baseline=-1mm
]
    \node[roundnode] (x) {$X$};
    \node[roundnode] (y) at (1, 0) {$Y$};

    \draw[-latex] (x) -- (y);
\end{tikzpicture}, which is a standard arrow in CATE (cfr. \cref{fig:ignorability}). Essentially, the potential outcome $Y(w)$ is the result of a natural process involving the fully observed $X$.\\
\begin{tikzpicture}[
    roundnode/.style={circle, draw=black,  minimum size=5mm, inner sep=0},
    baseline=-1mm
]
    \node[roundnode, fill=blue!15] (x_tilde) {$\widetilde{X}$};
    \node[roundnode] (w) at (1, 0) {$W$};

    \draw[-latex] (x_tilde) -- (w);
\end{tikzpicture}, when a clinician determines treatment, they have to make do with what is given to them, i.e. the observed covariate set $\tilde{X}$, which may be partially unobserved due to missingness. This proxy to $X$ is the best a clinician has available to them as the fully observed $X$ is not always available to them.\\
\begin{tikzpicture}[
    roundnode/.style={circle, draw=black,  minimum size=5mm, inner sep=0},
    baseline=-1mm
]
    \node[roundnode, fill=blue!15] (z) {$Z$};
    \node[roundnode] (w) at (1, 0) {$W$};

    \draw[-latex] (z) -- (w);
\end{tikzpicture}, throughout our work we consider $Z$ to be different from $\widetilde{X}$, where $\widetilde{X}$ entails the actual {\it values} of the covariate set, $Z$ indicates their presence. If or not a variable is present for a clinician to base their treatment decision on, can have an effect on their eventual decision. Say that it is too risky for a particular treatment when a patient's blood pressure is not observed, then a different treatment option will be chosen (or the variable will be measured before a decision is made). In this setting, the absence of a value has determined treatment, leading $Z$ into $W$.

Continuing our discussion, for a DAG to be valid as a missingness description, they {\it should not contain} the following:\\
\begin{tikzpicture}[
    roundnode/.style={circle, draw=black,  minimum size=5mm, inner sep=0},
    baseline=-1mm
]
    \node[roundnode] (x) {$X$};
    \node[roundnode] (w) at (1, 0) {$W$};

    \draw[-latex] (x) -- (w);
\end{tikzpicture}, as, again, $X$ is simply not available for a clinician to base their treatment-decision on.\\
\begin{tikzpicture}[
    roundnode/.style={circle, draw=black,  minimum size=5mm, inner sep=0},
    baseline=-1mm
]
    \node[roundnode, fill=blue!15] (x_tilde) {$\widetilde{X}$};
    \node[roundnode] (y) at (1, 0) {$Y$};

    \draw[-latex] (x_tilde) -- (y);
\end{tikzpicture}, as $Y(w)$ is the result of a natural process, it should depend on $X$, not $\widetilde{X}$.\\
\begin{tikzpicture}[
    roundnode/.style={circle, draw=black,  minimum size=5mm, inner sep=0},
    baseline=-1mm
]
    \node[roundnode, fill=blue!15] (z) {$Z$};
    \node[roundnode] (y) at (1, 0) {$Y$};

    \draw[-latex] (z) -- (y);
\end{tikzpicture}, similarly, $Y(w)$, should not depend on the missingness $Z$ as different datasets on the same person would register different outcomes when outcomes depend directly on $Z$, which cannot happen.

{\bf Permutations on CIT.} For CIT, there cannot be a direct arrow from $X$ to $W$ (a feat we agree with), as it automatically violates \cref{eq:CIT}. As long as there exists a path between $W$ and $Y$, excluding the connecting SWIG-path \citep{richardson2013single}, we consider the DAG valid. When excluding paths {\it going up} from $W$ or $Y(w)$, the total amount of valid DAGs amounts to $(C^1_3 + C^2_3 + C^3_3) \cdot (C^1_2 + C^2_2) = 21$, where $C_3$ is coming from three potential paths into $Y$ (having three variables, other than $W$ and $Y(w)$), and $C_2$ is coming from two potential paths into $W$ (three variables, excluding $X$). All the CIT DAGs, including the one presented in \citet{mayer2020doubly} are illustrated in \cref{fig:CIT:perm} (the DAGs presented in \citep{mayer2020doubly} correspond to \cref{fig:CIT:perm:original,fig:CIO:perm:original} for CIT and CIO, respectively).  

 \Cref{fig:CIT:perm} is organised as follows: each two rows contain 7 graphs (corresponding to $(C^1_3 + C^2_3 + C^3_3)$) where we vary the arrows going into $Y(w)$; there are three sets of graphs (where a set is two rows), for each set we vary the arrows going into $W$ (where three then corresponds to $(C^1_2 + C^2_2)$).

  Handy with \cref{tab:CIO_and_CIT_invalid}, we find that \cref{fig:CIT:perm:g} is the only valid DAG from all CIT (and later we see all CIO) compatible DAGs. Inspecting \cref{fig:CIT:perm:g} we notice that it corresponds exactly with \cref{fig:mcm}, without the bidirectional arrow between $W$ and $Z$, which is a result of informative versus uninformative missingness. Splitting $Z$ in $\out{Z}$ and $\into{Z}$ results automatically in MCM.

{\bf Permutations on CIO.} For CIO, the definition does not allow an arrow from $X$ to $Y(w)$, similarly to the definition of CIT not allowing an arrow from $X$ to $W$. Contrasting CIT, however, we find this restriction not sensible. As we have already argued above, $X$ is the {\it only} variable \quotes{justified} to influence the potential outcomes, $Y(w)$, directly. Any other variable influencing $Y(w)$ would result in contrasting outcomes over different datasets, which would imply, for example, different tumour sizes for the same person across different datasets. Having CIO and accompanying \cref{eq:CIO}, results in 21 (applying the same calculation as for CIT) non-sensible DAGs, including the DAG presented in \citet{mayer2020doubly}. All these DAGs are listed in \cref{fig:CIO:perm}, and evaluated (like for CIT) in \cref{tab:CIO_and_CIT_invalid}.  

\begin{table*}[h]
    \centering
    \caption{{\bf Validity of CIT and CIO.} There are six criteria we argue missingness in CATE should follow. These criteria are indicated in the column headers as directed dependencies missingness should, or should not, include. Assuming that there are no arrows {\it going up} from $W$ and $Y(w)$ (for example, the potential outcomes $Y(w)$, cannot influence the covariates $X$), we have 42 DAGs that respect \cref{eq:CATE:assum:UDM} and one of \cref{eq:CIT} or \cref{eq:CIO}. Only one of these DAGs (indicated in \colorbox{green!25}{green}), respects each criteria. This DAG (\cref{fig:CIT:perm:g}) is a permutation of CIT, and a version of MCM that does not assume factors, $\out{Z}$ and $\into{Z}$ in $Z$. In the below table, \quotes{\cmark} means presence and \quotes{\xxmark} means absence, when these icons are black absence or presence is positive, when they are gray they are negative.} 
    \label{tab:CIO_and_CIT_invalid}

        \caption{Original CIT}
        \label{fig:CIT:perm:original}
    \end{subfigure}%
    ~
    \begin{subfigure}{0.23\textwidth}
        \centering
        %
        \caption{}
        \label{fig:CIT:perm:b}
    \end{subfigure}%
    ~
    \begin{subfigure}{0.23\textwidth}
        \centering
        %
        \caption{}
        \label{fig:CIT:perm:c}
    \end{subfigure}%
    ~
    \begin{subfigure}{0.23\textwidth}
        \centering
        %
        \caption{}
        \label{fig:CIT:perm:d}
    \end{subfigure}%
    
    \begin{subfigure}{0.23\textwidth}
        \centering
        %
        \caption{}
        \label{fig:CIT:perm:e}
    \end{subfigure}%
    ~
    \begin{subfigure}{0.23\textwidth}
        \centering
        %
        \caption{}
        \label{fig:CIT:perm:f}
    \end{subfigure}%
    ~
    \begin{subfigure}{0.23\textwidth}
        \centering
        %
        \caption{}
        \label{fig:CIT:perm:g}
    \end{subfigure}%

    \begin{subfigure}{0.23\textwidth}
        \centering
        %
        \caption{}
        \label{fig:CIT:perm:h}
    \end{subfigure}%
    ~
    \begin{subfigure}{0.23\textwidth}
        \centering
        %
        \caption{}
        \label{fig:CIT:perm:i}
    \end{subfigure}%
    ~
    \begin{subfigure}{0.23\textwidth}
        \centering
        %
        \caption{}
        \label{fig:CIT:perm:j}
    \end{subfigure}%
    ~
    \begin{subfigure}{0.23\textwidth}
        \centering
        %
        \caption{}
        \label{fig:CIT:perm:k}
    \end{subfigure}%
    
    \begin{subfigure}{0.23\textwidth}
        \centering
        %
        \caption{}
        \label{fig:CIT:perm:l}
    \end{subfigure}%
    ~
    \begin{subfigure}{0.23\textwidth}
        \centering
        %
        \caption{}
        \label{fig:CIT:perm:m}
    \end{subfigure}%
    ~
    \begin{subfigure}{0.23\textwidth}
        \centering
        %
        \caption{}
        \label{fig:CIT:perm:n}
    \end{subfigure}%

    \begin{subfigure}{0.23\textwidth}
        \centering
        %
        \caption{}
        \label{fig:CIT:perm:o}
    \end{subfigure}%
    ~
    \begin{subfigure}{0.23\textwidth}
        \centering
        %
        \caption{}
        \label{fig:CIT:perm:p}
    \end{subfigure}%
    ~
    \begin{subfigure}{0.23\textwidth}
        \centering
        %
        \caption{}
        \label{fig:CIT:perm:q}
    \end{subfigure}%
    ~
    \begin{subfigure}{0.23\textwidth}
        \centering
        %
        \caption{}
        \label{fig:CIT:perm:r}
    \end{subfigure}%
    
    \begin{subfigure}{0.23\textwidth}
        \centering
        %
        \caption{}
        \label{fig:CIT:perm:s}
    \end{subfigure}%
    ~
    \begin{subfigure}{0.23\textwidth}
        \centering
        %
        \caption{}
        \label{fig:CIT:perm:t}
    \end{subfigure}%
    ~
    \begin{subfigure}{0.23\textwidth}
        \centering
        %
        \caption{}
        \label{fig:CIT:perm:u}
    \end{subfigure}%
    
    \caption{{\bf Permutations on CIT.} There are a total of 21 DAGs that respect \cref{eq:CATE:assum:UDM} and \cref{eq:CIT}. From these DAGs, only one --- \Cref{fig:CIT:perm:g} --- is acceptable. Add only the assumption that some elements in $Z$ are informative, and some elements in $Z$ are uninformative, and we automatically arrive at MCM. Note that we exclude edges {\it going up} from $W$ or $Y(w)$, similarly to \citet{richardson2013single}.}
    \label{fig:CIT:perm}
\end{figure}

\begin{figure}[h]
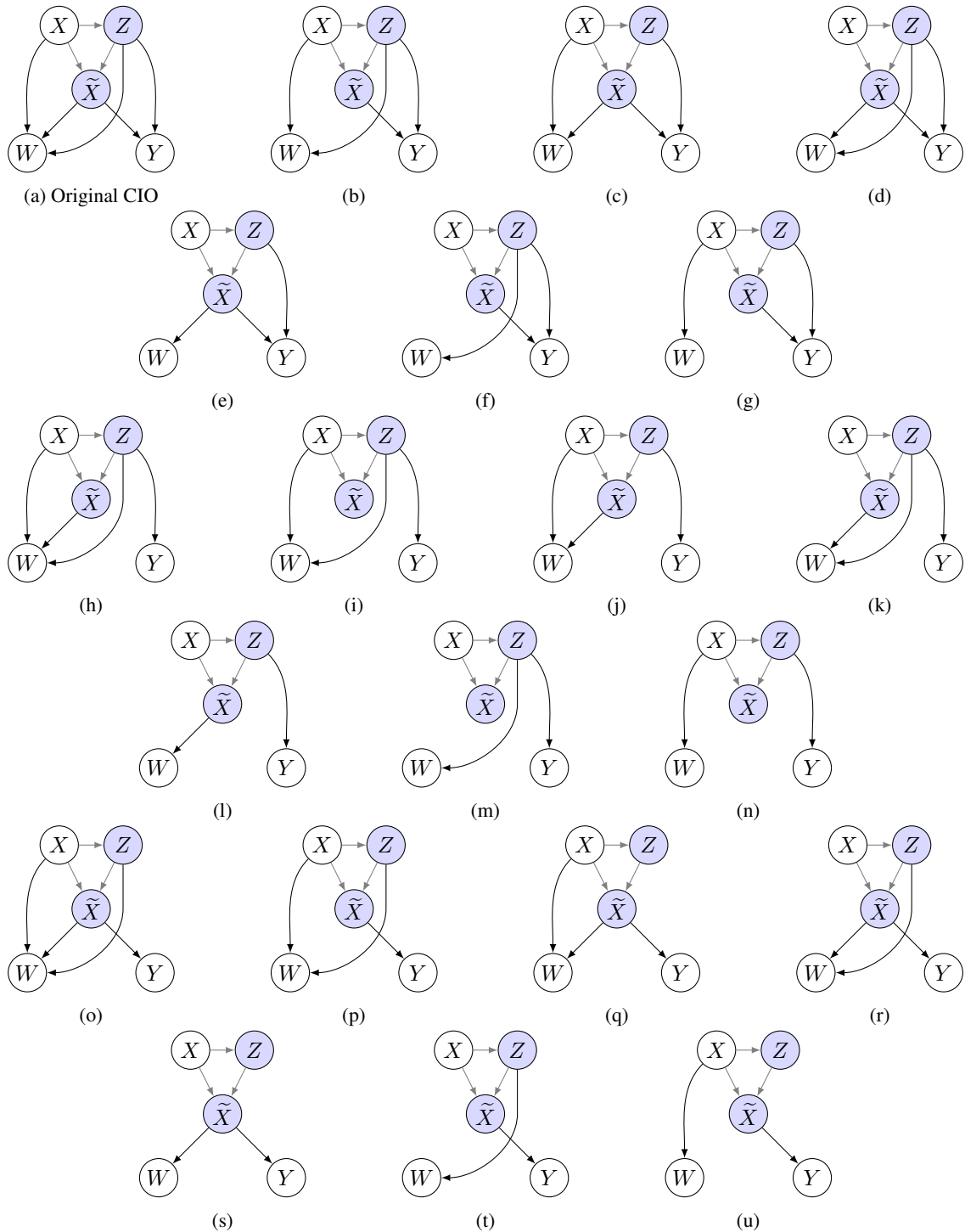

    \centering
    
    \begin{subfigure}{0.23\textwidth}
        \centering

        \caption{Original CIO}
        \label{fig:CIO:perm:original}
    \end{subfigure}%
    ~
    \begin{subfigure}{0.23\textwidth}
        \centering
        %
        \caption{}
        \label{fig:CIO:perm:b}
    \end{subfigure}%
    ~
    \begin{subfigure}{0.23\textwidth}
        \centering
        %
        \caption{}
        \label{fig:CIO:perm:c}
    \end{subfigure}%
    ~
    \begin{subfigure}{0.23\textwidth}
        \centering
        %
        \caption{}
        \label{fig:CIO:perm:d}
    \end{subfigure}%
    
    \begin{subfigure}{0.23\textwidth}
        \centering
        %
        \caption{}
        \label{fig:CIO:perm:e}
    \end{subfigure}%
    ~
    \begin{subfigure}{0.23\textwidth}
        \centering
        %
        \caption{}
        \label{fig:CIO:perm:f}
    \end{subfigure}%
    ~
    \begin{subfigure}{0.23\textwidth}
        \centering
        %
        \caption{}
        \label{fig:CIO:perm:g}
    \end{subfigure}%

    \begin{subfigure}{0.23\textwidth}
        \centering
        %
        \caption{}
        \label{fig:CIO:perm:h}
    \end{subfigure}%
    ~
    \begin{subfigure}{0.23\textwidth}
        \centering
        %
        \caption{}
        \label{fig:CIO:perm:i}
    \end{subfigure}%
    ~
    \begin{subfigure}{0.23\textwidth}
        \centering
        %
        \caption{}
        \label{fig:CIO:perm:j}
    \end{subfigure}%
    ~
    \begin{subfigure}{0.23\textwidth}
        \centering
        %
        \caption{}
        \label{fig:CIO:perm:k}
    \end{subfigure}%
    
    \begin{subfigure}{0.23\textwidth}
        \centering
        %
        \caption{}
        \label{fig:CIO:perm:l}
    \end{subfigure}%
    ~
    \begin{subfigure}{0.23\textwidth}
        \centering
        %
        \caption{}
        \label{fig:CIO:perm:m}
    \end{subfigure}%
    ~
    \begin{subfigure}{0.23\textwidth}
        \centering
        %
        \caption{}
        \label{fig:CIO:perm:n}
    \end{subfigure}%

    \begin{subfigure}{0.23\textwidth}
        \centering
        %
        \caption{}
        \label{fig:CIO:perm:o}
    \end{subfigure}%
    ~
    \begin{subfigure}{0.23\textwidth}
        \centering
        %
        \caption{}
        \label{fig:CIO:perm:p}
    \end{subfigure}%
    ~
    \begin{subfigure}{0.23\textwidth}
        \centering
        %
        \caption{}
        \label{fig:CIO:perm:q}
    \end{subfigure}%
    ~
    \begin{subfigure}{0.23\textwidth}
        \centering
        %
        \caption{}
        \label{fig:CIO:perm:r}
    \end{subfigure}%
    
    \begin{subfigure}{0.23\textwidth}
        \centering
        %
        \caption{}
        \label{fig:CIO:perm:s}
    \end{subfigure}%
    ~
    \begin{subfigure}{0.23\textwidth}
        \centering
        %
        \caption{}
        \label{fig:CIO:perm:t}
    \end{subfigure}%
    ~
    \begin{subfigure}{0.23\textwidth}
        \centering
        %
        \caption{}
        \label{fig:CIO:perm:u}
    \end{subfigure}%

    \caption{{\bf Permutations on CIO.} There are 21 possible DAGs that respect \cref{eq:CIO}. None of them is acceptable as they, by definition, cannot include a direct edge from $X$ to $Y(w)$. Having a direct edge between $X$ and $Y(w)$ encodes dependence, despite conditioning on $\widetilde{X}$ and $Z$. Note that we have excluded edges {\it going up} from $W$ or $Y(w)$, in similar fashion to \citet{richardson2013single}.}
    \label{fig:CIO:perm}
\end{figure}

\end{document}